\documentclass{article}

 \usepackage[main, final]{neurips_2026}

\usepackage[T1]{fontenc}    
\usepackage{hyperref}       
\usepackage{url}            
\usepackage{booktabs}       
\usepackage{amsfonts}       
\usepackage{nicefrac}       
\usepackage{microtype}      
\usepackage{xcolor}         
\usepackage{caption}
\usepackage{makecell}

\usepackage{amsmath}
\usepackage{amssymb}
\usepackage{mathtools}
\usepackage{amsthm}
\theoremstyle{plain}
\newtheorem{theorem}{Theorem}[section]

\theoremstyle{definition}
\newtheorem{definition}[theorem]{Definition}

\theoremstyle{remark}

\newcommand{\proj}{\texorpdfstring{MemR$^3$}{MemR3}}
\usepackage{multirow}
\usepackage{url}

\usepackage{graphicx}
\usepackage{float}
\definecolor{darkblue}{rgb}{0.1, 0.3, 0.7}
\definecolor{darkred}{rgb}{0.7,0.15,0.15}
\definecolor{darkgreen}{rgb}{0.1,0.5,0.1}

\newcommand{\gain}[1]{\textcolor{darkred}{\footnotesize\,(+#1)}}

\usepackage{enumitem}
\usepackage{threeparttable}
\usepackage{wrapfig}

\usepackage{tcolorbox}
\definecolor{viridis1}{RGB}{68, 1, 84}
\tcbuselibrary{breakable}
\tcbset{
  systemstyle/.style={
    colback=viridis1!10,      
    colframe=viridis1,        
    fonttitle=\scshape\bfseries, 
    boxrule=0.8mm,            
    sharp corners,            
    breakable                 
  },
}

\title{Knowing What is Missing: Efficient Conversational Memory via Explicit Evidence-Gap Tracking}

\author{
  \textbf{Xingbo Du}$^{1}$, \textbf{Loka Li}$^{1}$, \textbf{Duzhen Zhang}$^{1}$,
  \textbf{Leonard Song}$^{2}$, \textbf{Le Song}$^{1}$ \\
  {\normalfont $^{1}$Mohamed bin Zayed University of Artificial Intelligence, Abu Dhabi, UAE} \\
  {\normalfont $^{2}$The Westminster Schools, Atlanta, GA, USA}
}

\begin{document}

\maketitle

\begin{abstract}
Long-term conversational memory poses a fundamental system challenge for LLM agents: reasoning over all past interactions improves coverage but incurs prohibitive token costs and noise, while retrieval-based methods are efficient yet often fail on multi-hop and temporal questions. Existing multi-step retrievers partially address this gap, but typically operate over an ever-growing textual history, causing context expansion and noise accumulation across iterations. We propose {\proj}, a backend-agnostic closed-loop controller that reformulates conversational memory retrieval as explicit stateful decision-making. At each iteration, {\proj} maintains an \textit{evidence-gap state} consisting of grounded information (\textit{evidence}) and unresolved information requirements (\textit{gaps}), and uses this state to route among three actions: \texttt{retrieve}, \texttt{reflect}, and \texttt{answer}. Newly retrieved snippets are incorporated into the state and then masked from subsequent prompts, so the model conditions on a compact summary of progress rather than the full retrieval history. This design turns multi-step memory retrieval into a bounded, inspectable control process whose query-time context scales with the state instead of raw accumulated text. Experiments on LongMemEval$_s$ show that {\proj} surpasses implicit multi-step retrieval baselines and often approaches or exceeds full-context prompting while retrieving only 1\%-5\% of the full conversation tokens on long conversations. On LoCoMo, {\proj} consistently improves over its underlying RAG- and Zep-based memory backbones. These results suggest that explicit evidence-gap tracking is an effective abstraction for token-efficient conversational~\mbox{memory}.
\end{abstract}

\section{Introduction}
With recent advances in large language model (LLM) agents~\citep{zhong2023llm4eda,li2026bioxarena,du2026biodyad}, memory systems have become a key component for storing and retrieving long-term, personalized information. They can typically be categorized into two groups: 1) Parametric methods~\citep{wang2024wise,fang2025alphaedit} that encode memories implicitly into model parameters, which can handle specific knowledge but struggle with scalability and continual updates, as modifying parameters to incorporate new memories often risks catastrophic forgetting and requires expensive fine-tuning. 2) Non-parametric methods~\citep{xu2025amem,langmem_blog2025,chhikara2025mem0,rasmussen2025zep}, in contrast, store explicit external information, enabling flexible retrieval and continual augmentation.

However, as conversational histories grow to hundreds of thousands of tokens, using non-parametric stores exposes a fundamental system challenge: a trade-off between reasoning coverage and context efficiency. Feeding all available memories to the LLM (``Full-Context'') provides broad access to evidence but incurs high token cost and can suffer from context dilution (i.e., ``lost in the middle''). Conversely, traditional ``Retrieve-then-Answer'' paradigms (e.g., vanilla RAG) are token-efficient but often miss complex multi-hop and temporal dependencies because retrieval is performed in a single open-loop pass. While recent agentic retrieval systems~\citep{asai2024self,shinn2023reflexion,yao2023react,asl2025fair} can potentially bridge this gap via multi-step searching, many implementations rely on implicit reasoning traces and append raw retrieved chunks over iterations. In long-term conversational memory, this design causes context expansion and noise accumulation exactly when the agent most needs targeted evidence.

Orthogonal to storage-centric innovations, this paper proposes an agentic memory controller, {\proj} (i.e., \underline{Mem}ory \underline{R}etrieval system with \underline{R}eflective \underline{R}easoning), designed to improve this trade-off at query time. Instead of relying only on implicit chain-of-thought~\citep{zhang2025system}, {\proj} reformulates multi-step conversational memory retrieval as explicit stateful decision-making. Specifically, the system orchestrates a \texttt{router} that dynamically transitions among \texttt{retrieve}, \texttt{reflect}, and \texttt{answer} actions based on an explicit \textit{evidence-gap state}. This state represents the current epistemic status of the agent: grounded information (\textit{evidence}) and unresolved information requirements~(\textit{gaps}).

Crucially, {\proj} employs an ``absorb and discard'' mechanism. Once newly retrieved snippets are assimilated into the evidence-gap state, their raw texts are masked from future prompts. As a result, key decisions (when to stop, how to refine the next query, and when to reflect) are driven by a compact summary of progress rather than an accumulating context window. This turns multi-step memory retrieval into a bounded and inspectable control process whose query-time context scales with the compact state instead of raw accumulated text.

\begin{figure*}[tb!]
\centering
\includegraphics[width=\textwidth]{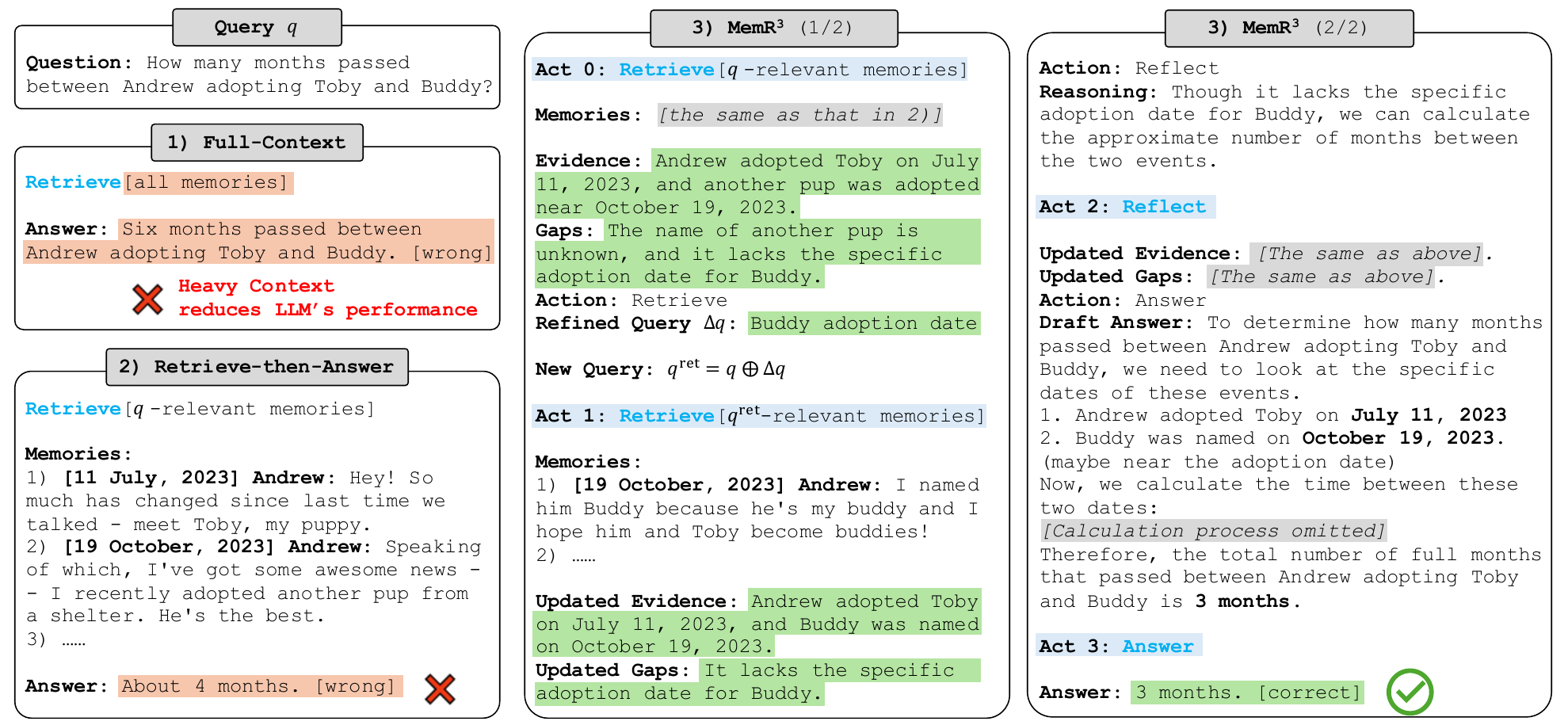}
\vspace{-10pt}
\caption{\label{fig:motivation} Illustration of three memory-usage paradigms. In this example, \textbf{Full-Context} exposes the LLM to many distractors, while \textbf{Retrieve-then-Answer} retrieves relevant snippets but still miscalculates. In contrast, \textbf{\proj} iteratively retrieves and reflects using an explicit evidence-gap tracker (Acts 0-3). By absorbing information and discarding raw texts, it refines the query and produces the correct answer (3 months) without unbounded context expansion.}
\vspace{-10pt}
\end{figure*}

The system offers three core \textbf{advantages}: 1) \textit{Token-Efficient Long-Horizon Reasoning.} By dynamically routing based on explicit gaps and masking assimilated raw snippets, {\proj} improves reasoning performance without allowing query-time context to grow with the full retrieval trajectory. 2) \textit{Backend-Agnostic Plug-and-Play.} As a stateful controller decoupled from the underlying storage mechanism, {\proj} can be integrated with existing memory databases (e.g., chunk-based or graph-based) without architectural modifications. 3) \textit{Explicit States and Diagnosability.} By explicitly isolating what is known from what is missing at each step, {\proj} shifts error attribution from black-box reasoning failures to structural diagnostics (e.g., coverage bias vs.\ hallucination bias), providing observability into the agent's decision boundary.
We compare {\proj}, the Full-Context setting, and the retrieve-then-answer paradigm in Fig.~\ref{fig:motivation}. The contributions are threefold:

(1) \textbf{A Stateful Closed-Loop Controller.}
We propose {\proj}, an autonomous controller that wraps existing memory stores and elevates standard retrieve-then-answer pipelines into a closed-loop process. This instantiates a principled, state-driven paradigm in which routing and query refinement are conditioned on an explicit epistemic state rather than on implicit reasoning.

(2) \textbf{Explicit Evidence-Gap Compression.}
We formalize an evidence-gap abstraction that summarizes the knowns and unknowns of a given query. This mechanism serves as a stochastic state estimator that compresses retrieval histories into a compact state whose size depends on the current summary rather than the full trajectory, explicitly isolating final reasoning from noisy raw text.

(3) \textbf{Empirical Evaluation on LongMemEval$_s$ and LoCoMo.}
We evaluate {\proj} on the long-horizon LongMemEval$_s$ benchmark, where it surpasses implicit multi-step retrieval baselines and often approaches or exceeds Full-Context prompting while retrieving only 1\%-5\% of the full conversation tokens. Furthermore, diagnostic evaluations on LoCoMo validate its plug-and-play capability across diverse memory backends (RAG and Zep), yielding consistent state-driven improvements over the underlying memory backbones.

\section{Related Work}
\textbf{Memory for LLM Agents}.
Prior work on non-parametric agent memory systems~\citep{wu2026rsmem} spans a wide range of fields, including management and utilization~\citep{du2025rethinking}, by storing structured~\citep{rasmussen2025zep} or unstructured~\citep{zhong2024memorybank} external knowledge. Specifically, production-oriented agents such as MemGPT~\citep{packer2023memgpt} introduce an OS-style hierarchical memory system that allows the model to page information between context and external storage, and SCM~\citep{wang2023enhancing} provides a controller-based memory stream that retrieves and summarizes past information only when necessary. Additionally, Zep~\citep{rasmussen2025zep} builds a temporal knowledge graph that unifies and retrieves evolving conversational and business data.
A-Mem~\citep{xu2025amem} creates self-organizing, Zettelkasten-style memory that links and evolves over time. Mem0~\citep{chhikara2025mem0} extracts and manages persistent conversational facts with optional graph-structured memory. Nemori~\citep{nan2025nemori} self-organizes long-term memory for LLM agents by segmenting conversations into episodes and learning new knowledge.
MIRIX~\citep{wang2025mirix} offers a multimodal memory system with six specialized memory types.  LightMem~\citep{fang2025lightmem} proposes a lightweight and efficient memory system inspired by the Atkinson-Shiffrin model. 

In this work, we focus on long-term conversational memory, leaving parametric or KV-cache mechanisms~\citep{wang2024wise,zhong2024memorybank,wei2025evo} out of scope. While recent non-parametric systems (e.g., Mem0, A-Mem, and Zep) optimize memory \textit{storage}, they often leave query-time control to open-loop or weakly stateful retrieval. Orthogonal to these storage-centric innovations, {\proj} acts as a backend-agnostic \textit{stateful controller}. By defining an explicit evidence-gap state, it can layer closed-loop, observable query-time control onto existing memory stores.

\textbf{Agentic Retrieval-Augmented Generation}.
RAG~\citep{lewis2020retrieval} established the modern retrieve-then-answer paradigm; subsequent work explored stronger retrievers~\citep{karpukhin2020dense,izacard2021leveraging}. Self-RAG~\citep{asai2024self}, Reflexion~\citep{shinn2023reflexion}, and ReAct~\citep{yao2023react} extend retrieval with reflection or iterative reasoning. FAIR-RAG~\citep{asl2025fair} refines queries using missing-information signals while retaining an accumulated filtered evidence pool. MemGuide~\citep{du2026memguide} uses intent-aligned retrieval and missing-slot tracking to select memories and guide user elicitation. {\proj} instead uses an explicit evidence-gap state to jointly control query refinement, retrieval, reflection, and stopping. Its ``absorb and discard'' mechanism masks assimilated raw text, keeping query-time prompts bounded by the current state and the most recent observations rather than by the full retrieval trajectory.

\begin{figure*}[tb!]
\centering
\includegraphics[width=0.98\textwidth]{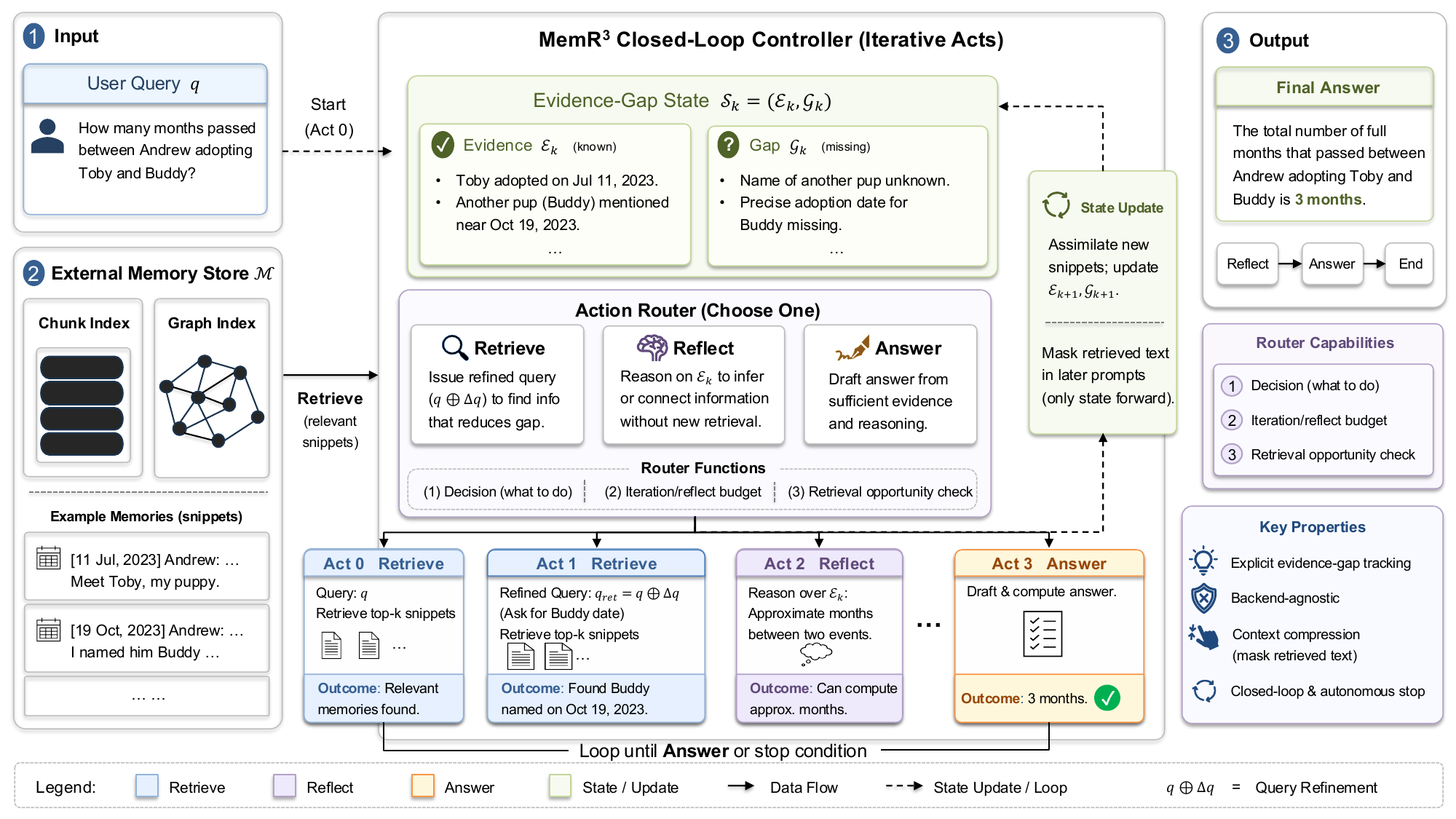}
\vspace{-5pt}
\caption{\label{fig:pipeline} Pipeline of {\proj}. {\proj} transforms retrieval into a closed-loop process: a router dynamically switches between Retrieve, Reflect, and Answer nodes while a global evidence-gap tracker maintains what is known and what is still missing. This enables iterative query refinement, targeted retrieval, and early stopping, making {\proj} an autonomous, backend-agnostic controller.}
\vspace{-10pt}
\end{figure*}

\section{Methodology}
\label{sec:method}
\subsection{From Implicit Reasoning to Explicit Markovian States}
Existing agentic retrieval systems (e.g., Self-RAG~\citep{asai2024self}, ReAct~\citep{yao2023react}) primarily treat multi-step retrieval as an implicit, autoregressive generation process. In these paradigms, the decision to retrieve, reflect, or answer relies entirely on the implicit chain-of-thought within the Large Language Model (LLM). While effective for short-horizon tasks, this approach is fundamentally non-Markovian: to decide the next action, the LLM must attend to an ever-growing trajectory of interleaved reasoning traces, past queries, and generated text. 

Since the epistemic boundary (what has been resolved vs. what remains missing) is entangled within this unstructured historical log, the agent cannot elegantly compress its past. In long-term memory scenarios, this reliance on an accumulating textual history inevitably leads to context dilution, error propagation (hallucination), and escalating token costs.

To overcome this, we reformulate long-horizon memory retrieval as a Stateful Control Paradigm (SCP). Instead of relying on implicit reasoning over an accumulating textual history, we equip the agent with an explicit epistemic state (belief state) to track its knowledge boundary dynamically. Under an idealized abstraction, this state is Markovian: the next routing decision depends on what has been grounded and what remains missing, rather than on the full textual trajectory. In practice, the state is an LLM-generated approximation, so its reliability must be evaluated empirically and diagnosed through failure modes.

\subsection{The Evidence-Gap State Machine: Bridging Theory and Practice}\label{sec:state_machine}
To formalize this stateful control, we first define an idealized state machine and then demonstrate how an LLM can be utilized to instantiate it.

\textbf{The Latent Requirement Space.} For any complex user query $q$, we conceptualize a latent Information Requirement Space $R(q) = \{r_1, r_2, ..., r_m\}$, which represents the minimal set of atomic facts (e.g., specific entities, timestamps, relations) necessary to completely answer $q$. The goal of the memory agent is to traverse the memory store $\mathcal{M}$ to fully satisfy $R(q)$.

\textbf{The Markovian Evidence-Gap Tracker.} At any iteration $k$, the agent maintains an explicit epistemic state $\mathcal{S}_k = (\mathcal{E}_k, \mathcal{G}_k)$.

\begin{itemize}[leftmargin=0.175in, nosep]
    \item \textbf{Evidence ($\mathcal{E}_k \subseteq R(q)$)}: The set of requirements that have been successfully grounded in the retrieved memory snippets up to step $k$.
    \item \textbf{Gap ($\mathcal{G}_k = R(q) \setminus \mathcal{E}_k$)}: The Dynamic Information Gap Tracker, representing the precise mathematical complement of the evidence. It explicitly defines the epistemic uncertainty, namely, what the agent still needs to know.
\end{itemize}

In the idealized formulation, $\mathcal{S}_k$ satisfies the Markov property and acts as a sufficient statistic of the retrieval history for the query. The router's policy $\pi(a_k | \mathcal{S}_k)$ then depends only on the current evidence and gap, allowing previously retrieved raw text to be omitted after assimilation. As shown in Theorem~\ref{thm:informal}, under an idealized extraction process (detailed in Appendix~\ref{sec:formal_eg}), this tracker follows a strict monotonic update rule: as new memory snippets are retrieved, $\mathcal{E}_k$ monotonically expands, and $\mathcal{G}_k$ monotonically shrinks, until the convergence $\mathcal{G}_k = \emptyset$ is met. The deployable system approximates this behavior with an LLM state estimator; consequently, the guarantees should be read as properties of the target abstraction, not as deterministic guarantees for every prompt execution.

\noindent
\begin{theorem}[\textbf{Informal}]\label{thm:informal}
Under an idealized requirement space $R(q)$ for a specific query $q$, the evidence-gap tracker in {\proj} is \textbf{monotone} (evidence never decreases and gaps never increase), \textbf{sound} (any requirement supported by the retrieved snippets is added to the evidence set when that support is observed), and \textbf{complete} (if the retrieval process eventually surfaces supporting snippets for all requirements in $R(q)$, then the ideal gap becomes empty at convergence).
\end{theorem}

In the example in Fig.~\ref{fig:pipeline}, $\mathcal{E}_k$ may contain ``Toby adopted on July 11, 2023'' and ``Buddy mentioned near October 19, 2023'', while $\mathcal{G}_k$ contains ``Buddy's precise adoption date is missing''. The empirical state stores textual evidence items and unresolved requirements rather than exact symbolic~\mbox{sets}.

\textbf{LLM as a Stochastic State Estimator.} In practice, the ideal requirement space $R(q)$ is not analytically given, and exact symbolic matching is intractable for unstructured text. Therefore, we instantiate this theoretical Markov state machine using an LLM. In our {\proj} system, the LLM serves as a stochastic state estimator.

At each iteration $k$, the LLM observes the newly retrieved snippets $\mathcal{P}_{k-1}$ and the previous state $(\mathcal{E}_{k-1}, \mathcal{G}_{k-1})$, and stochastically approximates the ideal transition:
\begin{equation}\label{eq:transition}
(\mathcal{E}_k, \mathcal{G}_k)
\mathrel{\approx}
(\mathcal{E}_k^\star, \mathcal{G}_k^\star)
\mathrel{=}
f_{\text{trans}}(\mathcal{P}_{k-1}, \mathcal{E}_{k-1}, \mathcal{G}_{k-1}),
\end{equation}
where $(\mathcal{E}_k^\star, \mathcal{G}_k^\star)$ is the ideal state, $f_{\text{trans}}$ denotes the idealized epistemic update rule-conceptually extracting newly verified facts from the observations $\mathcal{P}_{k-1}$ to monotonically expand the evidence set $\mathcal{E}$ and shrink the remaining gaps $\mathcal{G}$ (as formally defined in Appendix~\ref{sec:formal_eg}).

By framing the LLM's prompt-based extraction as a stochastic approximation of an ideal Markov state tracker, {\proj} provides a principled foundation for its routing mechanism. The router's actions (Retrieve, Reflect, Answer) are conditioned on the estimated state $(\mathcal{E}_k, \mathcal{G}_k)$ together with lightweight control mechanisms, rather than on a growing raw-text log. This decoupling of state estimation and action routing keeps the controller computationally efficient and makes its intermediate decisions inspectable even over extremely long dialogue horizons.

\subsection{{\proj} Architecture: Instantiating the Stateful Controller}\label{sec:architecture}

\begin{wraptable}{r}{0.41\textwidth}
\centering
\vspace{-13pt}
\captionsetup{type=algorithm}
\caption{Safety Overrides for the State-Conditioned Policy}
\label{alg:router}
\vspace{-5pt}
\setlength{\fboxsep}{4pt}
\fbox{%
\begin{minipage}{0.38\textwidth}
\footnotesize
\begin{algorithmic}[1]
\STATE \textbf{Input:} query $q$, snippets $\mathcal{P}_{k-1}$, iteration $k$, budgets $n_{\text{max}}, n_{\text{cap}}$, current reflect-streak length $n_{\text{reflect}}$, and original action $a_k$.
\STATE \textbf{Output:} updated action $a_k$, unchanged if no safety rail is triggered.
\IF{$k \ge n_{\text{max}}$}
    \STATE $a_k \gets \texttt{answer}$;
\ELSIF{$n_{\text{reflect}} \ge n_{\text{cap}}$}
    \STATE $a_k \gets \texttt{retrieve}$;
\ELSIF{$\mathcal{P}_{k-1} = \emptyset$}
    \STATE $a_k \gets \texttt{reflect}$;
\ENDIF
\STATE $k = k + 1$;
\end{algorithmic}
\end{minipage}%
}
\vspace{-15pt}
\end{wraptable}

To materialize the theoretical Markov state machine into a deployable system, we implement {\proj} as a directed agentic graph built upon LangGraph~\citep{langchain2025langgraph}. The architecture seamlessly decouples the control logic from the underlying storage, consisting of a central control node (the \texttt{Router}) and three operational nodes (\texttt{Retrieve}, \texttt{Reflect}, \texttt{Answer}). 

\textbf{Policy (Router).}
The LLM updates the evidence-gap state and proposes the next action conditioned on that state. Alg.~\ref{alg:router} specifies only deterministic safety overrides. In the implementation, these constraints are encoded as a decision directive before invoking the generator; otherwise, the generator selects among \texttt{retrieve}, \texttt{reflect}, and \texttt{answer} from the estimated state.


\textbf{Retrieve}. When the state indicates missing information, the agent transitions to the \texttt{Retrieve} node, generating a precise refinement query $\Delta q_k$ targeted exclusively at $\mathcal{G}_k$. Crucially, {\proj} treats the underlying memory store (e.g., chunk-based RAG or graph-based Zep) purely as an observation environment via plug-and-play integration. Once new snippets are assimilated into the state $\mathcal{S}_k$, the raw retrieved texts are explicitly masked from future contexts. This ``absorb and discard'' mechanism implements the compression discussed earlier. Retrieval excludes previously retrieved snippets by their stable identifiers and searches the remaining candidates. When no additional relevant candidate remains, it returns $\mathcal{P}=\emptyset$, triggering reflection rather than retrieving the same snippets again.

\textbf{Reflect}. The router autonomously transitions to the Reflect node when internal reasoning is deemed more effective than further retrieval. This is common when a retrieval action yields no new information, but the gap remains unresolved. Here, the agent performs internal reasoning to re-evaluate the collected evidence and potentially deduce answers without incurring external retrieval overhead.

\textbf{Answer}. Once the dynamic gap tracker converges ($\mathcal{G}_k = \emptyset$) or the iteration budget is exhausted, the system transitions to the \texttt{Answer} node. The final response is generated strictly grounded on the accumulated evidence state $\mathcal{E}_k$. This completely decouples the final reasoning step from the noisy, unbounded raw retrieval context that plagues traditional open-loop RAG systems.

A more detailed illustration of {\proj} architecture is shown in Appendix~\ref{appdx:architecture}.

\subsection{Theoretical and Practical Efficiency}\label{sec:efficiency}
A central dilemma in memory-augmented agents is context accumulation. Multi-step frameworks that append each retrieved chunk to the prompt can incur $O(N \times L)$ context growth over $N$ iterations, where $L$ is the chunk budget per step. This increases cost and can exacerbate ``lost in the middle'' behavior by flooding the LLM with retrieval noise. {\proj} mitigates this issue through stateful compression. By employing an ``absorb and discard'' mechanism, raw retrieved snippets are masked once grounded, carrying forward the semantic state $(\mathcal{E}_k, \mathcal{G}_k)$ and only the newest observations. Consequently, the query-time footprint is bounded by the size of the current evidence-gap state rather than by the full set of accumulated snippets. These architectural choices help {\proj} retain much of the reasoning robustness of Full-Context prompting while operating closer to the token cost of open-loop RAG.

\section{Experiments}
We first introduce the experimental setups in Sec.~\ref{sec:exp_protocol}. Main results include showing the capability of handling long-horizon memory on LongMemEval$_s$~\cite{wu2025longmemeval} (Sec.~\ref{sec:exp_longmemeval}) and the diagnostic evaluation on LoCoMo~\citep{maharana2024evaluating} (Sec.~\ref{sec:exp_locomo}). 
We then present a case study in Sec.~\ref{sec:case_study}. 
Other analyses in Sec.~\ref{sec:exp_others} include ablation studies, efficiency analysis, and the dataset analysis. 
Code available: \url{https://github.com/mbzuai-ai4bio/memr3}.

\subsection{Experimental Setup}\label{sec:exp_protocol}
\textbf{Datasets.} We evaluate {\proj} on two distinct benchmarks: 1) \textbf{LongMemEval$_s$}~\citep{wu2025longmemeval}, a long-horizon dataset comprising 500 questions paired with conversations of approximately 115k tokens on average, serving as the primary testbed for our efficiency-accuracy trade-off claims. 2) \textbf{LoCoMo}~\citep{maharana2024evaluating}, a fundamental diagnostic benchmark featuring complex question categories (multi-hop, temporal, open-domain, single-hop), averaging 26k tokens per conversation. Following existing work~\citep{chhikara2025mem0}, we exclude the adversarial category.

\textbf{Baselines as Paradigms.} 
To systematically assess our stateful controller, we categorize baselines into four paradigms: 1) \textit{Full-Memory Reference}: \textbf{Full-Context}, which provides a strong reference point when the entire conversation fits the context window~\citep{chhikara2025mem0}. 2) \textit{Open-loop Heuristics}: Vanilla \textbf{RAG}~\citep{lewis2020retrieval}. 3) \textit{Implicit Agentic Retrievers}: \textbf{Self-RAG}~\citep{asai2024self} and a custom strong baseline \textbf{RCR} (RAG-CoT-RAG) extended from ReAct~\citep{yao2023react}. 4) \textit{Memory Systems}: \textbf{A-Mem}~\citep{xu2025amem}, \textbf{LangMem}~\citep{langmem_blog2025}, \textbf{Mem0}~\citep{chhikara2025mem0}, and graph-based \textbf{Zep}~\citep{rasmussen2025zep}. We additionally compare with FAIR-RAG~\citep{asl2025fair}, one-shot query decomposition (Decomp-RAG), and iterative query rewriting (Rewrite-RAG), detailed in Appendix~\ref{appdx:baseline}.
\textit{Note on Evaluation Scope:} On LongMemEval$_s$, our primary comparison focuses on query-time controllers and full-context prompting. We do not include API-heavy memory-construction systems (e.g., Mem0, Zep) on this benchmark because building explicit indices for 115k-token conversations across all examples incurs substantial ingestion latency and cost. Instead, we include these comparisons on the LoCoMo~\mbox{benchmark}.


\textbf{Metrics \& Implementation Details.} We adopt the LLM-as-a-Judge (J) score as the metric following~\citet{chhikara2025mem0}, because surface-level metrics (e.g., F1, BLEU-1~\citep{xu2025amem, 10738994}) can penalize correct open-ended answers that do not lexically match the reference.
Appendix~\ref{appdx:repeated_eval} reports repeated-run statistics, Appendix~\ref{appdx:rebuttal_evaluation} provides EM/F1, and Sec.~\ref{sec:exp_others} reports a human audit of the evidence state. Total input-plus-completion costs are reported in Appendix~\ref{appdx:cost}. Retrieved-context tokens measure evidence volume and are distinct from total LLM usage across controller calls.
For LLM backends, we extensively use GPT-4o-mini~\citep{openai2024gpt4omini} for baseline tuning and GPT-4.1~\citep{openai2025gpt41} to judge whether the answer is correct according to the original question, the generated answer, and the gold answer, following the prompt by~\citet{chhikara2025mem0}. For long-context evaluations, we also test other frontier models (GPT-4.1-mini~\citep{openai2025gpt41}, GPT-5-mini~\citep{openai2025gpt5}, DeepSeek-V3.2~\citep{liu2025deepseek}, Claude-haiku-4.5~\citep{anthropic2025claude}, and Gemini-2.5-flash~\citep{comanici2025gemini}). We set $n_{\text{chk}} = 5$ for RAG and $n_{\text{max}} = 5$ for {\proj}. Detailed tuning processes are in Appendix~\ref{appdx:exp_other}.

\begin{figure*}[tb!]
\centering
\includegraphics[width=0.98\textwidth]{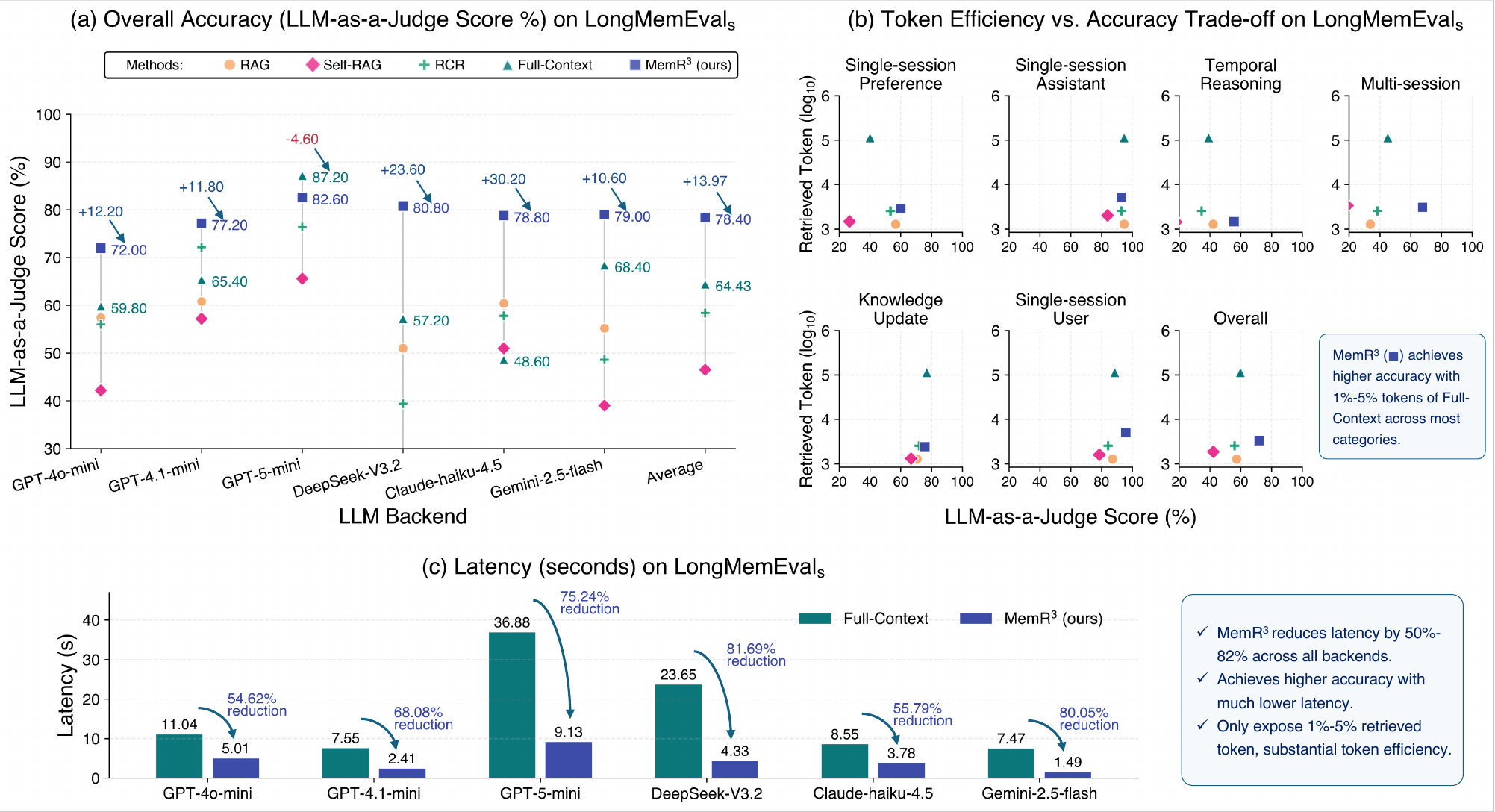}
\vspace{-5pt}
\caption{Results on LongMemEval$_s$. (a) Overall accuracy across six LLM backends. (b) Retrieved-context tokens versus accuracy across six categories. (c) LLM-response latency. Repeated-run statistics and full query-time costs are reported in Appendices~\ref{appdx:repeated_eval} and~\ref{appdx:cost}.}
\vspace{-5pt}
\label{fig:longmemeval}
\end{figure*}

\subsection{Improving the Efficiency-Accuracy Trade-off on Long-Horizon Memory
}\label{sec:exp_longmemeval}

To stress-test our core claim, we evaluate {\proj} against full-context prompting, one-shot RAG, and implicit multi-step retrievers on the long-horizon LongMemEval$_s$ benchmark across six LLM backends and six categories (Fig.~\ref{fig:longmemeval}).

\textbf{Limits of Implicit Context Accumulation.} 
The results in Fig.~\ref{fig:longmemeval} (a) reveal clear limitations of existing retrieval paradigms under long-horizon conversational memory. As expected, open-loop \textbf{Vanilla RAG} struggles to capture complex dependencies over sequences of up to 115k tokens. More importantly, implicit multi-step retrievers (\textbf{Self-RAG} and \textbf{RCR}) often perform worse than vanilla RAG. For example, Self-RAG reaches only 24.00\% under DeepSeek-V3.2 and 39.00\% under Gemini-2.5-flash. This trend is consistent with our analysis in Sec.~\ref{sec:efficiency}: when raw retrieved chunks are repeatedly appended to the prompt across multiple reasoning steps, the accumulated noise can outweigh the benefits of additional search.

\textbf{Stateful Compression Improves the Efficiency-Accuracy Trade-off.} 
In Fig.~\ref{fig:longmemeval} (a), {\proj} outperforms the retrieval baselines across all six backends and Full-Context on five of six. The exception is GPT-5-mini, where compression can drop temporal and multi-session details that this reasoning backend~\citep{openai2025gpt5} exploits from the full history (Table~\ref{tab:appdx_longmemeval}). On the stronger GPT-5.6 Terra, Full-Context also leads, while consuming about 16-17$\times$ more input tokens (Table~\ref{tab:backend_strength}). Full-Context accuracy thus depends jointly on conversation length and backend capability: very long histories may be difficult for some backends to exploit reliably, whereas stronger long-context backends can use them effectively. {\proj} therefore targets regimes where full-history inference is costly or unreliable, rather than claiming uniform accuracy superiority over Full-Context; when full-context inference is affordable and reliable, Full-Context remains a strong alternative.

The advantage of {\proj} is also reflected in the token-efficiency analysis in Fig.~\ref{fig:longmemeval} (b). Unlike implicit agentic retrievers that continuously accumulate raw retrieved text in the context window, {\proj} compresses newly acquired observations into an explicit state and masks assimilated text afterward. As a result, it occupies a more favorable high-accuracy, low-token region across categories, achieving stronger performance while consuming only 1\%-5\% of the retrieved-context tokens used by Full-Context prompting. These results suggest that long-horizon agents need not choose between one-shot retrieval and brute-force full-context reasoning; stateful compression offers a more effective middle ground.

\begin{table}[tb!]
\centering
\caption{Average end-to-end latency (seconds) on LongMemEval$_s$.}
\label{tab:latency_rebuttal}
\small
\begin{tabular}{l|ccccc}
\toprule
\textbf{Backend} & \textbf{RAG} & \textbf{Self-RAG} & \textbf{RCR} & \textbf{{\proj}} & \textbf{Full-Context} \\
\midrule
GPT-4o-mini & 1.53 & 7.06 & 2.81 & 3.87 & 11.04 \\
GPT-4.1-mini & 1.72 & 10.71 & 3.10 & 5.59 & 7.55 \\
GPT-5-mini & 11.51 & 41.38 & 15.80 & 15.26 & 36.88 \\
\bottomrule
\end{tabular}
\end{table}

\textbf{Latency Reduction.}
Across all three evaluated backends, {\proj} is faster than Full-Context and Self-RAG, with additional latency over one-shot RAG from iterative control (Table~\ref{tab:latency_rebuttal}).

\begin{table*}[tb!]
\centering
\caption{\label{tab:main}LLM-as-a-Judge Score (\%) on LoCoMo. Categories are grouped, with GPT-4o-mini (4o) and GPT-4.1-mini (4.1) shown side-by-side. $\Delta$ denotes absolute improvements over backbones. {\proj} rows are three-run means (Appendix~\ref{appdx:repeated_eval}); baselines are single~\mbox{runs}.}
\vspace{-5pt}
\resizebox{\textwidth}{!}{
\begin{tabular}{l|cc|cc|cc|cc|cc}
\toprule
\multirow{2}{*}{\textbf{Method}} 
& \multicolumn{2}{c|}{\textbf{1. Multi-Hop}} 
& \multicolumn{2}{c|}{\textbf{2. Temporal}} 
& \multicolumn{2}{c|}{\textbf{3. Open-Domain}} 
& \multicolumn{2}{c|}{\textbf{4. Single-Hop}} 
& \multicolumn{2}{c}{\textbf{Overall}} \\ \cmidrule{2-11}
& \textbf{4o} & \textbf{4.1} 
& \textbf{4o} & \textbf{4.1} 
& \textbf{4o} & \textbf{4.1} 
& \textbf{4o} & \textbf{4.1} 
& \textbf{4o} & \textbf{4.1} \\
\midrule

A-Mem~\citep{xu2025amem} 
& 61.70 & 71.99 & 64.49 & 74.77 & 40.62 & 58.33 & 76.63 & 79.88 & 69.06 & 76.00 \\

LangMem~\citep{langmem_blog2025} 
& 62.23 & 74.47 & 23.43 & 61.06 & 47.92 & 67.71 & 71.12 & 86.92 & 58.10 & 78.05 \\

Mem0~\citep{chhikara2025mem0} 
& 67.13 & 62.41 & 55.51 & 57.32 & 51.15 & 44.79 & 72.93 & 66.47 & 66.88 & 62.47 \\

\midrule

Self-RAG~\citep{asai2024self} 
& 69.15 & 75.89 & 64.80 & 75.08 & 34.38 & 54.17 & 88.31 & 90.12 & 76.46 & 82.08 \\

RCR 
& 71.28 & 80.85 & 71.03 & 81.62 & 42.71 & 62.50 & 86.99 & 90.12 & 77.96 & 84.89 \\

\midrule

Zep~\citep{rasmussen2025zep} 
& 67.38 & 72.34 & 73.83 & 77.26 & 63.54 & 64.58 & 78.67 & 83.49 & 74.62 & 78.94 \\

{\color{darkblue}\bfseries \proj} (Zep) 
& 69.39 & 77.78 & 73.83 & 77.78 & \textbf{67.01} & 69.79 & 80.60 & 84.42 & 76.26 & 80.88 \\
\quad $\Delta$ (abs.)
& \gain{2.01} & \gain{5.44} & \textcolor{gray}{\footnotesize (+0.00)} & \gain{0.52} 
& \gain{3.47} & \gain{5.21} & \gain{1.93} & \gain{0.93} 
& \gain{1.64} & \gain{1.94} \\

\midrule

RAG~\citep{lewis2020retrieval} 
& 68.79 & 73.05 & 65.11 & 73.52 & 58.33 & 62.50 & 83.86 & 85.90 & 75.54 & 79.46 \\

{\color{darkblue}\bfseries \proj} (RAG) 
& \textbf{71.39} & \textbf{81.20} 
& \textbf{76.22} & \textbf{82.14} 
& 61.11 & \textbf{71.53} 
& \textbf{89.44} & \textbf{92.17} 
& \textbf{81.55} & \textbf{86.75} \\
\quad $\Delta$ (abs.)
& \gain{2.60} & \gain{8.15} 
& \gain{11.11} & \gain{8.62} 
& \gain{2.78} & \gain{9.03} 
& \gain{5.58} & \gain{6.27} 
& \gain{6.01} & \gain{7.29} \\
\midrule

Full-Context 
& 72.34 & 86.43 & 58.88 & 86.82 & 59.38 & 71.88 & 86.39 & 93.73 & 76.32 & 89.00 \\

\bottomrule
\end{tabular}
}
\vspace{-10pt}
\end{table*}

\subsection{Diagnostic Evaluation on LoCoMo}\label{sec:exp_locomo}
Having established its long-context efficiency, we utilize the LoCoMo benchmark to explicitly quantify how {\proj}'s stateful controller improves complex reasoning across different question typologies and memory backbones (Table~\ref{tab:main}).

\textbf{Backend-Agnostic Improvements.} Across both LLM backends, {\proj} consistently improves the performance of its underlying open-loop retrievers (RAG and Zep), supporting its utility as a plug-and-play controller. Most notably, under GPT-4o-mini, the RAG-backed {\proj} achieves an overall score of 81.55\% (+6.01\% over vanilla RAG), surpassing all retrieval baselines and Full-Context prompting (76.32\%). Under GPT-4.1-mini, {\proj} yields a +7.29\% overall improvement over RAG, suggesting that the benefits of explicit state tracking persist with stronger base models.

\textbf{Multi-hop \& Temporal Reasoning.} Multi-hop and temporal questions directly stress a system's ability to coordinate distant memories and avoid noisy over-retrieval. 
For \textit{temporal queries}, naively providing the entire conversation can hamper reasoning due to timeline confusion, as reflected by Full-Context's 58.88\% under GPT-4o-mini. {\proj} isolates specific timestamps by tracking explicit gaps, reaching 76.22\% (+11.11\% over RAG). 
For \textit{multi-hop queries}, {\proj} explicitly coordinates iterative steps without losing track of intermediate conclusions. Under GPT-4.1-mini, it achieves 81.20\% (+8.15\% over RAG), slightly exceeding the strong, multi-step RCR baseline (80.85\%) without suffering from the latter's context accumulation issues.

\textbf{Open-Domain \& Single-hop.} Open-domain questions demand diverse background knowledge, making retrieval harder to trigger accurately. Here, {\proj} benefits from reformulating queries when initial evidence is noisy. Under GPT-4o-mini, the Zep-backed {\proj} achieves the highest category score (67.01\%), while under GPT-4.1-mini, the RAG-backed variant improves by +9.03\% to 71.53\%, nearly matching Full-Context (71.88\%). For simpler \textit{single-hop} queries, {\proj} yields strong scores (e.g., 92.17\% under GPT-4.1-mini), suggesting that the state-driven early stopping mechanism can terminate queries without over-complicating the reasoning trace.

\begin{figure*}[tb!]
\centering
\includegraphics[width=\textwidth]{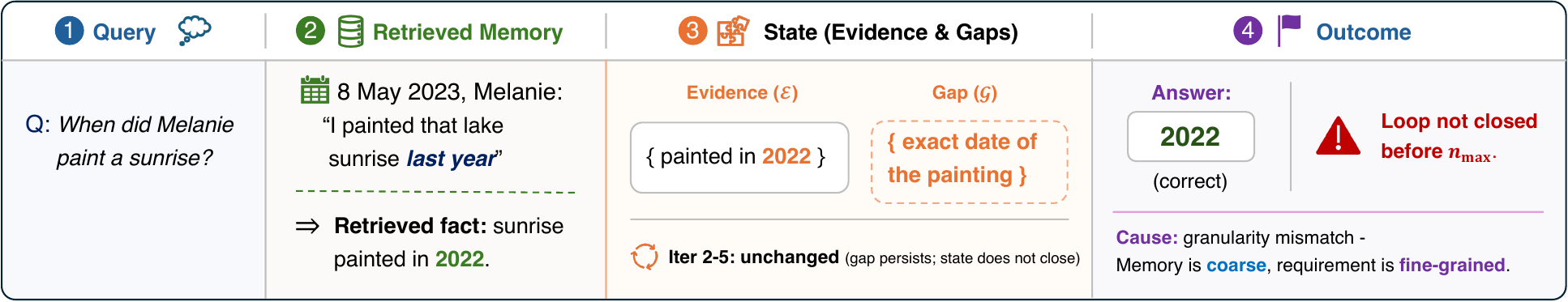}
\vspace{-10pt}
\caption{\label{fig:failed_case_in_main} \textbf{Granularity bias (failed case).} The agent recovers the coarse fact ``in 2022'' and produces the correct answer, but the gap requires the exact date. Because of the granularity mismatch, the controller cannot close the loop and stops only at the iteration limit $n_{\text{max}}$.}
\vspace{-10pt}
\end{figure*}

\subsection{Case Study: A Failed Case with Granularity Bias}\label{sec:case_study}
In Fig.~\ref{fig:failed_case_in_main}, the agent successfully retrieves a coarse-grained fact (``painted in 2022'') and produces the correct answer, yet fails to close the loop within the iteration budget. The issue arises from a mismatch between the granularity of the latent requirement and the available evidence: while the memory supports a year-level timestamp, the gap tracker implicitly maintains a finer-grained requirement (e.g., exact date). As a result, the evidence set is sufficient for answering but insufficient for satisfying the residual gap, causing the state to stagnate across iterations. This case highlights that, beyond retrieval coverage or hallucination, granularity mismatch can prevent convergence even when the answer is correct. Importantly, the explicit evidence-gap decomposition makes such failure modes observable: the persistent gap directly reveals that the bottleneck lies in requirement granularity rather than missing retrieval or reasoning errors.
The same mismatch bounds open-ended, high-density queries. The state can represent multiple answer-relevant requirements, but selectors such as ``major'' in ``summarize all major life changes'' are underspecified and may introduce granularity mismatch or incomplete coverage. Our claim is that state size scales with the selected answer-relevant requirements rather than the full raw history, not that {\proj} guarantees exhaustive coverage for subjective, high-density summaries.

For other failed cases, such as coverage bias (false negatives) and hallucination bias (false positives), we further present their case studies in Appendix~\ref{appdx:case_study}. 

\subsection{Other analyses (evaluated on LoCoMo)}\label{sec:exp_others}

\begin{wraptable}{l}{0.48\textwidth}
\centering
\vspace{-12pt}
\begin{minipage}[tb!]{1\linewidth}
\centering
\caption{\label{tab:abaltion}Ablation studies on LoCoMo.}
\vspace{-6pt}
\resizebox{\textwidth}{!}{
\begin{threeparttable}
\begin{tabular}{l|cccc|c}
\toprule
\textbf{Method} & \textbf{MH}* & \textbf{Temporal} & \textbf{OD}* & \textbf{SH}* & \textbf{Overall} \\
\midrule\midrule
RAG & 68.79 & 65.11 & 58.33 & 83.86 & 75.54 \\
{\proj} & \textbf{71.39} & \textbf{76.22} & 61.11 & \textbf{89.44} & \textbf{81.55} \\
~ w/o mask & 62.41 & 68.54 & 55.21 & 72.17 & 68.54 \\
~ w/o $\Delta q_k$ & 66.67 & 75.08 & 60.42 & 83.37 & 77.11 \\
~ w/o reflect & 65.25 & 73.83 & \textbf{61.46} & 83.37 & 76.65 \\
\bottomrule
\end{tabular}
\begin{tablenotes}
\item[*] MH = Multi-hop; OD = Open-domain; SH = Single-hop. {\proj}: three-run mean.
\end{tablenotes}
\end{threeparttable}
}
\end{minipage}
\vspace{-10pt}
\end{wraptable}

\textbf{Ablation Studies.} We progressively disable core components of the RAG-backed {\proj} (Table~\ref{tab:abaltion}). Disabling the ``absorb and discard'' masking (\textit{w/o mask}) causes a severe degradation (81.55\% $\rightarrow$ 68.54\%), consistent with interference from repeatedly exposed raw memories. This variant still answers from the compressed state, so retaining raw history can add interference without recovering omitted evidence. Removing the state-driven refinement query (\textit{w/o $\Delta q_k$}) drops single-hop and multi-hop scores significantly, while disabling reflection (\textit{w/o reflect}) removes the agent's ability to reason over existing evidence, harming overall performance (81.55\% $\rightarrow$ 76.65\%).

\begin{table}[tb!]
\centering
\caption{Human audit of retrieval and evidence preservation on 152 LoCoMo questions (GPT-4o-mini). $R=1$: retrieved snippets contain answer-critical evidence; $E=1$: the final state faithfully retains it.}
\label{tab:state_audit}
\small
\begin{tabular}{cc|l|cc}
\toprule
$\boldsymbol{R}$ & $\boldsymbol{E}$ & \textbf{Interpretation} & \textbf{Questions} & \textbf{Human accuracy (\%)} \\
\midrule
0 & 0 & Retrieval failure & 23 & 0.0 \\
1 & 0 & State/compression failure & 24 & 4.2 \\
1 & 1 & Evidence preserved & 105 & 90.5 \\
\bottomrule
\end{tabular}
\vspace{-10pt}
\end{table}

\textbf{State Quality.} We manually audit 152 questions from LoCoMo's first conversation under GPT-4o-mini (Table~\ref{tab:state_audit}). The state preserves answer-critical evidence in 105 of 129 retrieval-success cases (81.4\%). Replacing it with a human-verified oracle state built from the same retrieved snippets raises human answer accuracy from 63.2\% to 84.2\%, identifying state estimation as a remaining~\mbox{bottleneck}.

\begin{table}[tb!]
\centering
\caption{Cumulative evidence recall and final J score (\%) on LoCoMo, grouped by initial retrieval~\mbox{recall}.}
\label{tab:initial_recall}
\resizebox{\textwidth}{!}{
\begin{tabular}{l|lc|ccccc|c}
\toprule
\multirow{2}{*}{\textbf{Backend}} & \multirow{2}{*}{\textbf{Subset}} & \multirow{2}{*}{$\boldsymbol{n}$} & \multicolumn{5}{c|}{\textbf{Cumulative recall}} & \multirow{2}{*}{\textbf{J}} \\ \cmidrule{4-8}
 & & & \textbf{Initial} & \textbf{Iter. 2} & \textbf{Iter. 3} & \textbf{Iter. 4} & \textbf{Iter. 5} & \\
\midrule
\multirow{3}{*}{GPT-4o-mini} & $R_0=0$ & 192 & 0.00 & 16.11 & 24.33 & 29.77 & 33.90 & 45.83 \\
 & $0<R_0<1$ & 206 & 48.00 & 52.49 & 56.92 & 58.66 & 59.99 & 70.87 \\
 & $R_0=1$ & 1,124 & 100.00 & 100.00 & 100.00 & 100.00 & 100.00 & 90.04 \\
\midrule
\multirow{3}{*}{GPT-4.1-mini} & $R_0=0$ & 192 & 0.00 & 21.84 & 34.16 & 41.50 & 47.58 & 64.58 \\
 & $0<R_0<1$ & 206 & 48.00 & 54.72 & 59.54 & 63.32 & 65.23 & 78.64 \\
 & $R_0=1$ & 1,124 & 100.00 & 100.00 & 100.00 & 100.00 & 100.00 & 91.99 \\
\bottomrule
\end{tabular}
}
\vspace{-10pt}
\end{table}

\textbf{Initial Retrieval Quality.} To examine sensitivity to initial retrieval, we count a gold supporting dialogue as covered when retrieved chunks jointly cover at least 50\% of its characters. Cumulative recall is the fraction of gold dialogues covered by an iteration. Table~\ref{tab:initial_recall} groups questions by initial recall $R_0$: iterative retrieval partially recovers initial misses, although initial coverage remains strongly associated with final accuracy.

\textbf{Revisiting the Evaluation Protocols of LoCoMo.}
During our reproduction of the baselines, we identified an ambiguity in the LoCoMo dataset's category indexing. Specifically, the mapping between numerical IDs and semantic categories (e.g., Multi-hop vs. Single-hop) requires care because the category names are not directly encoded in the released question files. To avoid conflating categories, we realign the evaluation protocol using category-level question counts and apply the same alignment to all baselines in Table~\ref{tab:main}. Detailed realignment is illustrated in Appendix~\ref{appdx:data_misalign}.

\section{Conclusion}
In this work, we introduce {\proj}, a stateful closed-loop memory controller for long-horizon conversational agents that improves the trade-off between reasoning capability and context efficiency. By modeling multi-step retrieval as explicit evidence-gap tracking, {\proj} isolates what is known from what remains missing, while its ``absorb and discard'' mechanism keeps query-time state bounded. Experiments show that {\proj} improves over its underlying memory backbones, outperforms implicit multi-step retrievers, and often matches or surpasses Full-Context prompting with far fewer tokens. It also makes intermediate memory decisions inspectable, enabling diagnosis of coverage, hallucination, and granularity failures. More capable long-term memory may improve personalized assistants, but may also raise privacy risks if sensitive user information is retrieved.

\textbf{Limitations.}
{\proj} still has several limitations: 1) its performance is bounded by the recall of the underlying memory store; 2) the current system lacks a rollback mechanism, so errors or omissions introduced during state compression may persist after the supporting raw snippets are masked; we leave provenance-guided replay as future work; 3) open-ended requests such as ``summarize all major life changes'' may suffer from granularity mismatch or incomplete coverage (Sec.~\ref{sec:case_study}); 4) in LongMemEval$_s$, we do not include ingestion-heavy systems such as Mem0, A-Mem, or Zep due to indexing cost, so we leave large-scale comparisons with such systems to future work; and 5) the current state machine is limited to text-based conversational memory and does not yet extend to multimodal settings~\citep{zhang2024mm,li2026personax}.


\begin{ack}
We thank the anonymous reviewers and the area chair for their constructive comments and suggestions.
Le Song is supported by funds from Mohamed bin Zayed University of Artificial Intelligence
(MBZUAI) and Weizmann Institute of Science (WIS) projects.
\end{ack}

\bibliography{ref}
\bibliographystyle{neurips_2026}

\newpage








\appendix

\section{Prompts}

\subsection{System prompt of the \texttt{generate} node}\label{appdx:prompt_system}
The system prompt is defined as follows, where the ``decision\_directive'' instructs the maximum iteration budget, reflect-streak capacity, and retrieval opportunity check, introduced in Sec.~\ref{sec:architecture}. Generally, ``decision\_directive'' is a textual instruction:
\textit{``reflect'' if you need to think about the evidence and gaps; choose ``answer'' ONLY when evidence is solid and no gaps are noted; choose ``retrieve'' otherwise.} However, when the maximum iterations budget is reached, ``decision\_directive'' is set as ``answer'' to stop early. When the reflection reaches the maximum capacity, ``decision\_directive'' is set as ``retrieve'' to avoid repeated ineffective reflection. When no useful retrieval remains, ``decision\_directive'' is set as ``reflect'' to avoid repeated ineffective retrieval. Through these constraints, the agent can avoid infinite ineffective actions to maintain stability.

{\footnotesize
\begin{tcolorbox}[systemstyle, title=System Prompt]
You are a memory agent that plans how to gather evidence before producing the final response shown to the user.\\
Always reply with a strict JSON object using this schema:\\
- evidence: JSON array of concise factual bullet strings relevant to the user's question; preserve key numbers/names/time references. If exact values are unavailable, include the most specific verified information (year/range) without speculation. Never mention missing or absent information here - ``gaps'' will do that.\\
- gaps: gaps between the question and evidence that prevent a complete answer.\\
- decision: one of [``retrieve'',``answer'',``reflect'']. Choose \textcolor{blue}{\{decision\_directive\}}.
\\ \\
Only include these conditional keys:\\
- retrieval\_query: only when decision == ``retrieve''. Provide a STANDALONE search string; short (5-15 tokens).\\
\hspace*{1em}* BAD Query: ``the date'' (lacks context).\\
\hspace*{1em}* GOOD Query: ``graduation ceremony date'' (specific).\\
\hspace*{1em}* STRATEGY: \\
\hspace*{2em}1. Search for the ANCHOR EVENT. (e.g. Question: ``What happened 2 days after X?'', Query: ``timestamp of event X'').\\
\hspace*{2em}2. Search for the MAPPED ENTITY. (e.g. Question: ``Weather in the Windy City'', Query: ``weather in Chicago'').\\
- detailed\_answer: only when decision == ``answer''; response using current evidence (keep absolute dates, avoid speculation). If evidence is limited, provide only what is known, or make cautious inferences grounded solely in that limited evidence. Do not mention missing or absent information in this field.\\
- reasoning: only when decision == ``reflect''; if further retrieval is unlikely, use current evidence to think step by step through the evidence and gaps, and work toward the answer, including any time normalization.
\\ \\
Never include extra keys or any text outside the JSON object.
\end{tcolorbox}
}

\subsection{User prompt of the \texttt{generate} node}\label{appdx:prompt_user}
Apart from the system, the user prompt is responsible for feeding additional information to the LLM. Specifically, at iteration $k+1$, ``question'' is the original question $q$. ``evidence\_block'' and ``gap\_block'' are evidence $\mathcal{E}_k$ and gaps $\mathcal{G}_k$ introduced in Sec.~\ref{sec:state_machine}. ``raw\_block'' is the retrieved raw snippets $\mathcal{P}_k$. ``reasoning\_block'' is the reasoning content. ``last\_query'' is the refined query $\Delta q_k$ introduced in Sec.~\ref{sec:architecture} that enables the new query to be different from the prior one. Note that these fields can be left empty if the corresponding information is not present.

{\footnotesize
\begin{tcolorbox}[systemstyle, title=User Prompt]
\# Question\\
\textcolor{blue}{\{question\}}
\\ \\
\# Evidence\\
\textcolor{blue}{\{evidence\_block\}}
\\ \\
\# Gaps\\
\textcolor{blue}{\{gap\_block\}}
\\ \\
\# Memory snippets\\
\textcolor{blue}{\{raw\_block\}}
\\ \\
\# Reasoning\\
\textcolor{blue}{\{reasoning\_block\}}
\\ \\
\# Prior Query\\
\textcolor{blue}{\{last\_query\}}
\\ \\
\# INSTRUCTIONS:\\
1. Update the evidence as a JSON ARRAY of concise factual bullets that directly help answer the question (preserve key numbers/names/time references; use the most specific verified detail without speculation).\\
2. Update gaps: remove resolved items, add new missing specifics blocking a full answer, and set to ``None'' when nothing is missing.\\
3. If you produce a retrieval\_query, make sure it differs from the previous query.\\
4. Decide the next action and return ONLY the JSON object described in the system prompt.
\end{tcolorbox}
}

\section{Formalizing the Markovian Evidence-Gap Tracker}
\label{sec:formal_eg}
A central component of {\proj} is the stateful evidence-gap tracker introduced in Sec.~\ref{sec:state_machine}, which formulates the retrieval process as a Markov Decision Process (MDP). This tracker maintains an evolving, compressed epistemic state of i) what information has been reliably established from memory ($\mathcal{E}$) and ii) what specific requirements remain analytically unresolved ($\mathcal{G}$).
While the deployable system instantiates this tracker via LLM-generated states, we introduce an idealized formal abstraction below. This abstraction clarifies the intended Markovian transition, enables principled analysis, and provides a theoretical foundation for diagnosing system robustness (e.g., isolating retrieval coverage bias from LLM hallucination bias).

\begin{definition}[Idealized Requirement Space]
\label{def:req_space}
For a user query $q$, we define a finite set of \emph{atomic information requirements}, which specify the minimal facts needed to fully answer the query:
\begin{equation}
R(q) = \{ r_1, r_2, \dots, r_m \}.
\end{equation}
\end{definition}

For example, for the question
``How many months passed between events $A$ and $B$?'', the requirement set can be
\begin{equation}
R(q) = \{\text{date}(A), \text{date}(B)\}.
\end{equation}

Each $r \in R(q)$ is associated with a symbolic predicate (e.g., a timestamp, entity attribute, or event relation), and $R(q)$ provides the semantic target against which retrieved memories are judged.
\begin{definition}[Memory-Support Relation]
\label{def:mem_support}
Let $\mathcal{M}$ be the memory store and $\mathcal{P}_k \subseteq \mathcal{M}$ denote the snippets retrieved at iteration $k$.
We define a relation $m \models r$ to indicate that memory item $m \in \mathcal{M}$ contains sufficient information to support requirement $r \in R(q)$.
Formally, $m \models r$ holds if the textual content of $m$ contains a minimal witness (e.g., a timestamp, entity mention, or explicit assertion) matching the predicate corresponding to $r$.
The matching criterion may be implemented via deterministic pattern rules or LLM-based semantic matching.
\end{definition}

\begin{definition}[Idealized Markovian State Transition]
\label{def:eg_update}
At iteration $k$, the idealized system maintains an explicit epistemic state $\mathcal{S}_k^\star = (\mathcal{E}_k^\star, \mathcal{G}_k^\star)$.
Given newly retrieved snippets $\mathcal{P}_{k-1}$ from the previous routing action, the ideal deterministic state transition is:
\begin{equation}
\mathcal{E}_k^\star = \mathcal{E}_{k-1}^\star \cup \left\{ r \in R(q) \big| \exists m \in \mathcal{P}_{k-1} \text{ s.t. } m \models r \right\},
\qquad
\mathcal{G}_k^\star = R(q) \setminus \mathcal{E}_k^\star.
\end{equation}
\end{definition}

Crucially, this update rule strictly satisfies the Markov property: the new state $\mathcal{S}_k^\star$ depends only on the immediate prior state $\mathcal{S}_{k-1}^\star$ and the newly observed snippets $\mathcal{P}_{k-1}$. The system monotonically accumulates verified requirements and strictly shrinks the gap. Because the state acts as a sufficient statistic, the raw snippets $\mathcal{P}_{k-1}$ can be safely purged once the state is updated without information loss, thereby preventing token usage from growing with the accumulated retrieval history, under the state abstraction adopted in this work.

\subsection{Practical Instantiation via LLM as a Stochastic State Estimator}\label{appdx:appx_theorm}
In the deployable {\proj} system, exact symbolic matching for the memory-support relation ($m \models r$) is intractable for unstructured text. Therefore, the theoretical Markov transition is instantiated through the LLM, acting as a \textit{stochastic state estimator}:
\begin{equation}
\mathcal{E}_k, \mathcal{G}_k = \texttt{LLM}(q, \mathcal{P}_{k-1}, \mathcal{E}_{k-1}, \mathcal{G}_{k-1}, \text{prompt}),
\end{equation}
where the prompt explicitly instructs the model to:

(i) extract concise factual evidence from $\mathcal{P}_{k-1}$ to update $\mathcal{E}_{k-1}$,

(ii) re-evaluate and enumerate the missing information to form $\mathcal{G}_k$, and

(iii) strictly avoid speculative inference.

Thus, the empirical state $(\mathcal{E}_k, \mathcal{G}_k)$ serves as a stochastic approximation to the idealized state $(\mathcal{E}_k^\star, \mathcal{G}_k^\star)$:
\begin{equation}
(\mathcal{E}_k, \mathcal{G}_k) \approx f_{\text{trans}}(\mathcal{P}_{k-1}, \mathcal{E}_{k-1}, \mathcal{G}_{k-1}).
\end{equation}

Deviations from the idealized state arise primarily from LLM extraction noise (e.g., hallucination bias) or semantic granularity mismatches between the text and the latent requirement space. This state-estimation perspective directly informs our diagnostic ablation and error attribution in the main experiments, bridging our rigorous theoretical abstraction with prompt-driven empirical behavior.

\subsection{Correctness Properties under Idealized Extraction}\label{appdx:theory}
Although the practical LLM instantiation lacks deterministic guarantees, the idealized Markovian tracker defined in Definition~\ref{def:eg_update} satisfies several intuitive properties essential for reliable closed-loop retrieval. 

\begin{theorem}[Properties of the Idealized Tracker]
\label{thm:eg_properties}
Assume the initialization $\mathcal{E}_0^\star = \varnothing$ (and hence $\mathcal{G}_0^\star = R(q)$).
Under the deterministic state transition rule (Definition~\ref{def:eg_update}), the following hold:
\begin{enumerate}
    \item \textbf{Monotonicity:}
    $\mathcal{E}_{k-1}^\star \subseteq \mathcal{E}_k^\star$ and $\mathcal{G}_k^\star \subseteq \mathcal{G}_{k-1}^\star$ for all $k \ge 1$.
    \item \textbf{Soundness:}
    If $m \models r$ for some newly retrieved memory $m \in \mathcal{P}_{k-1}$, then $r \in \mathcal{E}_k^\star$.
    \item \textbf{Completeness at convergence:}
    If every requirement $r \in R(q)$ is supported by the accumulated trajectory of retrieved snippets (i.e., $\exists m \in \bigcup_{0 \le j \le K-1} \mathcal{P}_j$ such that $m \models r$), then the recursive state updates ensure $\mathcal{E}_K^\star = R(q)$ and hence the gap is closed, $\mathcal{G}_K^\star = \varnothing$.
\end{enumerate}
\end{theorem}

\begin{proof}
(1) By the Markovian transition rule in Definition~\ref{def:eg_update},
\begin{equation}
    \mathcal{E}_k^\star = \mathcal{E}_{k-1}^\star \cup 
    \big\{ r \in R(q) \,\big|\, \exists m \in \mathcal{P}_{k-1} \text{ s.t. } m \models r \big\},
\end{equation}
so it trivially holds that $\mathcal{E}_{k-1}^\star \subseteq \mathcal{E}_k^\star$.
Since $\mathcal{G}_k^\star = R(q) \setminus \mathcal{E}_k^\star$ and $\mathcal{E}_{k-1}^\star \subseteq \mathcal{E}_k^\star$, we obtain $\mathcal{G}_k^\star \subseteq \mathcal{G}_{k-1}^\star$.

(2) If $m \models r$ for some $m \in \mathcal{P}_{k-1}$, then by Definition~\ref{def:eg_update} we have $r \in \{ r' \in R(q) \mid \exists m' \in \mathcal{P}_{k-1} \text{ s.t. } m' \models r' \} \subseteq \mathcal{E}_k^\star$.

(3) Let $r \in R(q)$ be arbitrary. By assumption, there exists
some retrieved snippet $m \in \bigcup_{0 \le j \le K-1} \mathcal{P}_j$
such that $m \models r$. Hence, for some iteration $j \in \{0,\ldots,K-1\}$,
we have $m \in \mathcal{P}_j$ and $m \models r$. By the transition rule in
Definition~\ref{def:eg_update}, this support is incorporated at the next state update, so
$r \in \mathcal{E}_{j+1}^\star$.
By the monotonicity property established in Part~(1),
\begin{equation}
\mathcal{E}_{j+1}^\star
\subseteq
\mathcal{E}_{j+2}^\star
\subseteq
\cdots
\subseteq
\mathcal{E}_K^\star .
\end{equation}

Therefore, $r \in \mathcal{E}_K^\star$. Since $r \in R(q)$ was arbitrary,
we have $R(q) \subseteq \mathcal{E}_K^\star$. Conversely,
$\mathcal{E}_K^\star \subseteq R(q)$ by construction of the update rule.
Thus $\mathcal{E}_K^\star = R(q)$, and consequently
\begin{equation}
\mathcal{G}_K^\star = R(q) \setminus \mathcal{E}_K^\star = \emptyset .
\end{equation}
\end{proof}

These properties characterize the target Markovian behavior that the stochastic LLM-based tracker aims to approximate, clarifying why a correct state estimate is sufficient to keep context bounded while reasoning converges.


\subsection{Approximation Biases and System Limitations}

The abstraction in this section assumes access to an ideal tracker that updates ($\mathcal{E}_k^\star$, $\mathcal{G}_k^\star$) exactly according to the requirement-support relation $m \models r$. In practice, {\proj} uses an LLM-generated tracker ($\mathcal{E}_k$, $\mathcal{G}_k$), which only approximates this ideal update. 
This introduces several forms of approximation bias: i) \textit{Coverage bias} (false negatives): supported requirements $r \in R(q)$ that are omitted from $\mathcal{E}_k$; ii) \textit{Hallucination bias} (false positives): requirements $r$ that appear in $\mathcal{E}_k$ even though no retrieved memory item supports them. Additionally, there could be other system limitations like: iii) \textit{Granularity bias}: cases where the tracker records a coarser fact (e.g., a year) but the requirement space $R(q)$ contains a finer predicate (e.g., an exact date), so the ideal requirement is never fully satisfied. We provide a more detailed discussion of these biases in the case studies in Sec.~\ref{appdx:case_study}.

\section{Implementation Details of the Stateful Controller}
\label{appdx:architecture}
To instantiate the theoretical Markovian state machine proposed in Section~\ref{sec:state_machine}, we explicitly define a directed graph using the LangGraph framework~\citep{langchain2025langgraph}. The graph consists of several structural nodes (\texttt{start}, \texttt{end}) and functional nodes (\texttt{generate}, \texttt{router}, \texttt{retrieve}, \texttt{reflect}, \texttt{answer}). Specifically, \texttt{start} is unconditionally followed by \texttt{retrieve}, and \texttt{end} is reached immediately after \texttt{answer}. The \texttt{generate} node acts as the stochastic state estimator introduced in Eq.~\ref{eq:transition}. Below, we detail the \texttt{router} policy and the three action execution nodes.

\textbf{Router (The Control Policy).} 
At each iteration $k$, the \texttt{router} acts as an autonomous sequential controller. It evaluates the current epistemic state and selects the next optimal action from $\{\texttt{retrieve}, \texttt{reflect}, \texttt{answer}\}$. Each action $a_k$ is accompanied by a specific textual generation conditioned on the state:
\begin{equation*}
    a_k \in \{(\texttt{retrieve}, \Delta q_k), (\texttt{reflect}, f_k), (\texttt{answer}, w_k)\},
\end{equation*}
where $\Delta q_k$ is a targeted refinement query addressing the specific gap, $f_k$ is an internal reasoning trace, and $w_k$ is a draft answer. They are subsequently utilized by the downstream action nodes. 

To ensure practical stability and prevent infinite loops in the state machine, the \texttt{router} applies three deterministic safety constraints (Algorithm~\ref{alg:router}): 1) a maximum iteration budget $n_{\text{max}}$ that forces termination (\texttt{answer}) once exhausted, 2) a reflect-streak capacity $n_{\text{cap}}$ that forces exploration (\texttt{retrieve}) when too many consecutive internal reflections occur, and 3) a retrieval-opportunity check that switches the action to \texttt{reflect} if the observation environment returns no new snippets.

These lightweight rules act as necessary safety rails while preserving operational flexibility. The primary routing signal remains the evolving evidence-gap state, which mathematically determines whether additional retrieval is necessary to close the gap, whether reflection is warranted, and when answering becomes justified.

\textbf{Retrieve (The Observation Action).}
When transitioning to the \texttt{retrieve} node, the system uses the generated refinement $\Delta q_k$ to construct $q_k^{\mathrm{ret}} = q \oplus \Delta q_k$ (where $\oplus$ denotes string concatenation), fetching new memory snippets from the environment:
\begin{equation}\label{eq:retrieve}
\mathcal{P}_k = \texttt{Retrieve}(q_k^{\mathrm{ret}}, \mathcal{M}\setminus \mathcal{M}^{\text{ret}}_{k-1}), ~
\mathcal{M}^{\text{ret}}_k = \mathcal{M}^{\text{ret}}_{k-1} \cup \mathcal{P}_k.
\end{equation}

Crucially, this node implements our ``absorb and discard'' mechanism. The snippets $\mathcal{P}_k$ are evaluated independently during the next-generation step \textit{without} accumulating the historical raw text $\mathcal{M}^{\text{ret}}_{k-1}$, thereby keeping the prompt footprint bounded by the current state and newest snippets. Furthermore, subtracting $\mathcal{M}^{\text{ret}}_{k-1}$ masks previously retrieved snippets, preventing redundant selection and improving the signal-to-noise ratio per iteration.

A major architectural benefit of {\proj} is that it treats all concrete retrievers as decoupled observation environments. Any underlying storage, be it vector search, graph memory (e.g., Zep), hybrid stores, or future paradigms, can be seamlessly integrated, provided they return textual snippets and support stable identifiers for masking. This abstraction ensures {\proj} remains a lightweight, portable, and strictly backend-agnostic controller.

\textbf{Reflect (Internal State Re-evaluation).}
Related refinement methods include confidence-guided self-correction~\citep{li2024confidence} and selective decoding~\citep{tang2025reflection}.
The \texttt{reflect} node allows the agent to update its epistemic state without querying the external memory store. It incorporates the reasoning trace $f_{k-1}$ and invokes the generator to re-evaluate the state $(\mathcal{E}_k, \mathcal{G}_k)$ via Eq.~\ref{eq:transition}, enabling the system to deduce implicit conclusions from the existing evidence and dynamically reformulate~\mbox{gaps}.

\textbf{Answer (Terminal Convergence).}
Once the router determines the gap is closed (or safely terminates via budget constraints), the state machine converges to the \texttt{answer} node. The final response is generated using the original query $q$, the draft answer $w_k$, and the dense semantic evidence $\mathcal{E}_k$, utilizing the prompt $p_{w}$ adapted from \citet{rasmussen2025zep}:
\begin{equation}
    w \leftarrow \texttt{LLM}(q, w_k, \mathcal{E}_k, p_{w}).
\end{equation}
By grounding the final generation strictly on the filtered evidence $\mathcal{E}_k$ rather than the unbounded raw retrieval context, the LLM is naturally insulated from noisy text. It is explicitly instructed to avoid hallucinations and remain rigorously faithful to the extracted facts.

\section{Experimental Settings}
\subsection{Licenses for Datasets}\label{appdx:licence}
We use two existing conversational memory benchmarks in our experiments:
LongMemEval$_s$~\citep{wu2025longmemeval} and LoCoMo~\citep{maharana2024evaluating}.
LongMemEval$_s$ is released under the MIT License according to its public dataset
repository.\footnote{\url{https://huggingface.co/datasets/xiaowu0162/longmemeval}}
LoCoMo is released under the Creative Commons Attribution-NonCommercial 4.0
International License (CC-BY-NC-4.0) according to its public dataset
repository.\footnote{\url{https://huggingface.co/datasets/adymaharana/locomo}}
We use both datasets only for research evaluation, cite the original dataset
papers, and do not redistribute modified versions of the datasets.

\subsection{Baselines}\label{appdx:baseline}
As shown in Sec.~\ref{sec:exp_protocol}, we select four groups of methods: 1) \textit{Full-Memory Reference}: \textbf{Full-Context}. 2) \textit{Open-loop Heuristics}: Vanilla \textbf{RAG}~\citep{lewis2020retrieval}. 3) \textit{Implicit Agentic Retrievers}: \textbf{Self-RAG}~\citep{asai2024self} and a custom strong baseline \textbf{RCR} (RAG-CoT-RAG) extended from ReAct~\citep{yao2023react}. 4) \textit{Memory Systems}: \textbf{A-Mem}~\citep{xu2025amem}, \textbf{LangMem}~\citep{langmem_blog2025}, \textbf{Mem0}~\citep{chhikara2025mem0}, and graph-based \textbf{Zep}~\citep{rasmussen2025zep}. 

\subsubsection{Full-Memory Reference}
We include \textbf{Full-Context} as a strong reference baseline, which provides the model with the entire conversation or memory buffer without retrieval. This setting is unconstrained by retrieval recall, but it can still suffer from context dilution when the relevant evidence is buried in long histories. This setting directly follows \citet{chhikara2025mem0,wang2025mirix}.

\subsubsection{Open-loop Heuristics}
\textbf{Vanilla RAG}~\citep{lewis2020retrieval}. The vanilla RAG retrieves the top-$k$ relevant snippets from the query once and provides a direct answer, without iterative retrieval or reasoning-based refinement. Other retrieval settings ($n_{\text{chk}}$, chunk size, etc.) are the same as those in {\proj}.

\subsubsection{Implicit Agentic Retrievers}
In this group, we examine the agentic structures underlying memory retrieval to demonstrate {\proj}'s advanced performance and, in particular, to highlight the advantage of {\proj}'s agentic structure. To validate this, we include Self-RAG~\citep{asai2024self} and design a strong heuristic baseline, RAG-CoT-RAG (RCR), which combines RAG and CoT~\citep{wei2022chain}.

\textbf{Self-RAG}~\citep{asai2024self}. A model-driven retrieval/reflect controller where the LLM decides, at each step, whether to answer or issue a refined retrieval query. Unlike MemR3, retrieval decisions in Self-RAG are implicit in the model's chain-of-thought, without explicit state tracking. We reproduce their original code and prompt to suit our task.

\textbf{FAIR-RAG}~\citep{asl2025fair}. FAIR-RAG decomposes the question into sub-queries, filters the retrieved evidence, and refines the queries until a structured evidence assessment judges the evidence sufficient. As no official code is released, we reimplement it from the published prompts with the same chunks, embeddings, and five-iteration budget as {\proj}.

\textbf{Decomp-RAG}. A single LLM call decomposes the question into subqueries, each subquery is retrieved once, and the union of the results is used to answer. The query set is fixed before observing retrieved evidence.

\textbf{Rewrite-RAG}. Each round rewrites the query conditioned on previously retrieved raw documents. This baseline has no persistent evidence-gap state or state-based stopping criterion.

\textbf{RAG-CoT-RAG (RCR)}. We design a strong heuristic baseline that extends beyond ReAct~\citep{yao2023react} by performing one initial retrieval~\citep{lewis2020retrieval}, a CoT~\citep{wei2022chain} step to identify missing information, and a second retrieval using a refined query. It provides multi-step retrieval but lacks an explicit evidence-gap state or a general controller.

\subsubsection{Memory Systems}
In this group, we consider recent advanced memory systems, including \textbf{A-Mem}~\citep{xu2025amem}, \textbf{LangMem}~\citep{langmem_blog2025}, \textbf{Mem0}~\citep{chhikara2025mem0}, and Zep~\citep{rasmussen2025zep} to demonstrate {\proj}'s overall effectiveness from a memory control perspective. 

\textbf{A-Mem}~\citep{xu2025amem}\footnote{\url{https://github.com/WujiangXu/A-mem}}. A-Mem is an agent memory module that turns interactions into atomic notes and links them into a Zettelkasten-style graph using embeddings plus LLM-based linking.

\textbf{LangMem}~\citep{langmem_blog2025}. LangMem is LangChain's persistent memory layer that extracts key facts from dialogues and stores them in a vector store (e.g., FAISS/Chroma) for later~\mbox{retrieval}.

\textbf{Mem0}~\citep{chhikara2025mem0}\footnote{\url{https://github.com/mem0ai/mem0}}. Mem0 is an open-source memory system that enables an LLM to incrementally summarize, deduplicate, and store factual snippets, with an optional graph-based memory extension.

\textbf{Zep}~\citep{rasmussen2025zep}. Zep is a hosted memory service that builds a time-aware knowledge graph over conversations and metadata to support fast semantic and temporal queries. We implement their original code and the official Zep service/API.

\subsection{Experimental protocols.}\label{appdx:exp_other}
The experiments are conducted on a machine with an AMD EPYC 7713P 64-core processor, an A100-SXM4-80GB GPU, and 512GB of RAM. Repeated evaluations use three independent executions under the same configuration.

We first describe how we determine the number of chunks $n_{\text{chk}}$ (for RAG backbone) and the number of relevant facts/entities (for Zep backbone) $n_{\text{fact}}$ per iteration, as well as the maximum iteration budget $n_{\text{max}}$.
Specifically, we choose a nominal configuration for {\proj} (with a RAG retriever) by arbitrarily setting the number of chunks per iteration $n_{\text{chk}} = 3$ and the max iteration budget $n_{\text{max}} = 5$. In Fig.~\ref{fig:params} a), we fix $n_{\text{max}} = 5$ and perform ablations over $n_{\text{chk}} \in \{1, 3, 5, 7, 9\}$. In Fig.~\ref{fig:params} b), we also fix $n_{\text{max}} = 5$ and perform ablations over $n_{\text{fact}} \in \{5, 10, 15, 20\}$. 
In Fig.~\ref{fig:params} c), we fix $n_{\text{chk}} = 3$ and perform ablations over $n_{\text{max}} \in \{1, 2, 3, 4, 5\}$. Considering both the LLM-as-a-Judge score and token usage, we eventually choose $n_{\text{chk}} = 5$, $n_{\text{fact}} = 15$, and $n_{\text{max}} = 5$ in all main experiments.

\begin{figure}[tb!]
\centering
\includegraphics[width=0.78\textwidth]{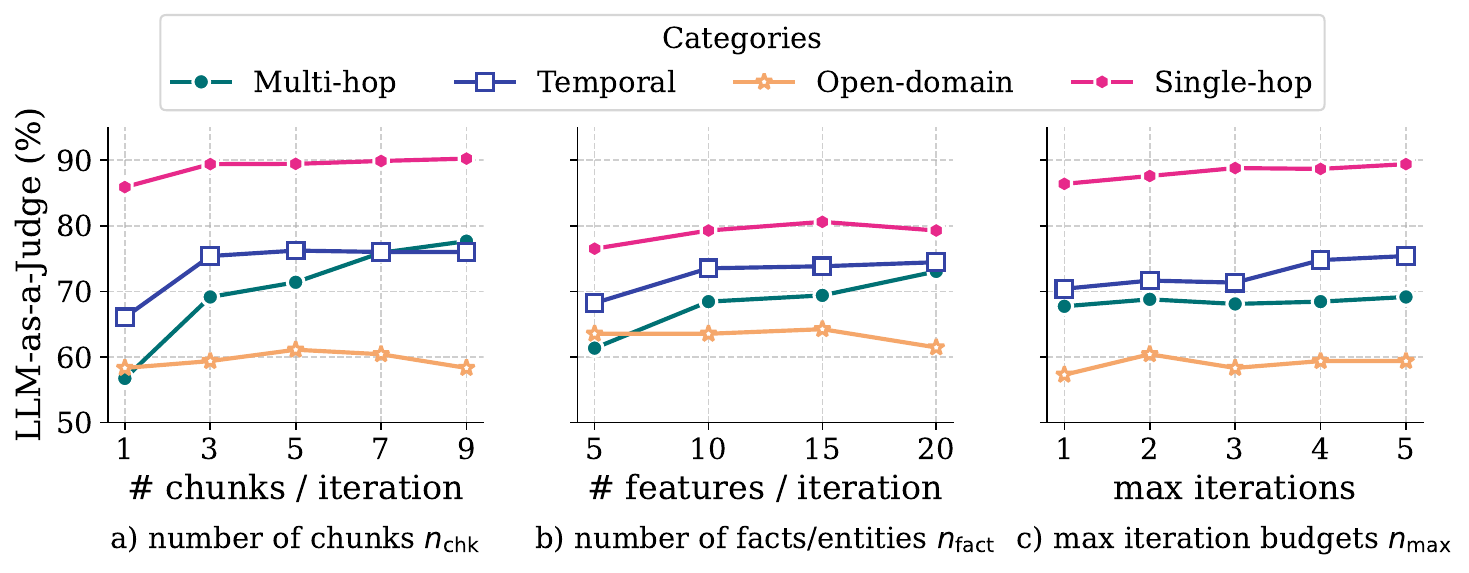}
\caption{\label{fig:params}LLM-as-a-Judge score (\%) with different a) number of chunks per iteration, b) number of facts/entities, and c) max iterations.}
\vspace{-10pt}
\end{figure}

For all chunk-based methods like RAG~\citep{lewis2020retrieval}, Self-RAG~\citep{asai2024self}, RAG-CoT-RAG, and {\proj} (RAG retriever), we set the embedding model as \texttt{text-embedding-3-large}~\citep{openai2024embeddinglarge3} and use a re-ranking strategy~\citep{reimers2019sentence} (\texttt{ms-marco-MiniLM-L-12-v2}) to search relevant memories rather than just similar ones. The chunk size is selected from \{128, 256, 512, 1024\} using GPT-4o-mini backend when $n_{\text{max}}=1$ and $n_{\text{chk}}=1$, and we ultimately choose 256, which is also in line with Mem0~\citep{chhikara2025mem0}.

\raggedbottom
\section{Experimental Results}

\subsection{Repeated Experiments}\label{appdx:repeated_eval}
Tables~\ref{tab:baseline_longmem}-\ref{tab:fair_locomo} report three-run means and sample standard deviations ($n=3$). Individual runs are listed in Tables~\ref{tab:baseline_longmem}, \ref{tab:fair_longmem}, and~\ref{tab:repeated}. Unless stated otherwise, single-run LongMemEval$_s$ results elsewhere use Run 1, the original~\mbox{evaluation}.

\begin{table}[H]
\centering
\caption{Three-run J scores (\%) on LongMemEval$_s$ with GPT-4o-mini.}
\label{tab:baseline_longmem}
\small
\begin{tabular}{l|cccc|c}
\toprule
 & \textbf{RAG} & \textbf{Self-RAG} & \textbf{RCR} & \textbf{Full-Context} & \textbf{{\proj}} \\
\midrule
Mean $\pm$ SD & $56.60 \pm 0.72$ & $41.53 \pm 0.61$ & $56.07 \pm 0.31$ & $58.53 \pm 1.10$ & $71.67 \pm 0.76$ \\
\midrule
Run 1 & 57.40 & 42.20 & 56.00 & 59.80 & 72.00 \\
Run 2 & 56.00 & 41.00 & 55.80 & 57.80 & 70.80 \\
Run 3 & 56.40 & 41.40 & 56.40 & 58.00 & 72.20 \\
\bottomrule
\end{tabular}
\end{table}

\begin{table}[H]
\centering
\caption{Three-run comparison with FAIR-RAG on LongMemEval$_s$ across six backends. The last column lists individual~\mbox{runs}.}
\label{tab:fair_longmem}
\resizebox{\textwidth}{!}{
\begin{tabular}{l|l|cc|c}
\toprule
\textbf{Backend} & \textbf{Method} & \textbf{J (\%)} & \textbf{Completion tokens} & \textbf{Individual runs (J)} \\
\midrule
\multirow{2}{*}{GPT-4o-mini} & FAIR-RAG & $54.93 \pm 0.46$ & $996 \pm 4$ & 55.20 / 55.20 / 54.40 \\
 & {\proj} & $71.67 \pm 0.76$ & $357 \pm 74$ & 72.00 / 70.80 / 72.20 \\
\midrule
\multirow{2}{*}{GPT-4.1-mini} & FAIR-RAG & $60.00 \pm 0.92$ & $1,060 \pm 29$ & 60.80 / 60.20 / 59.00 \\
 & {\proj} & $76.40 \pm 1.06$ & $491 \pm 5$ & 77.20 / 76.80 / 75.20 \\
\midrule
\multirow{2}{*}{GPT-5-mini} & FAIR-RAG & $73.20 \pm 1.40$ & $10,144 \pm 185$ & 74.60 / 73.20 / 71.80 \\
 & {\proj} & $81.80 \pm 0.72$ & $5,029 \pm 78$ & 82.60 / 81.20 / 81.60 \\
\midrule
\multirow{2}{*}{DeepSeek-V3.2} & FAIR-RAG & $49.60 \pm 0.40$ & $1,931 \pm 11$ & 49.20 / 49.60 / 50.00 \\
 & {\proj} & $80.47 \pm 1.14$ & $815 \pm 136$ & 80.80 / 79.20 / 81.40 \\
\midrule
\multirow{2}{*}{Claude-haiku-4.5} & FAIR-RAG & $66.27 \pm 1.17$ & $3,671 \pm 116$ & 65.40 / 65.80 / 67.60 \\
 & {\proj} & $78.80 \pm 0.40$ & $1,019 \pm 3$ & 78.80 / 78.40 / 79.20 \\
\midrule
\multirow{2}{*}{Gemini-2.5-flash} & FAIR-RAG & $61.47 \pm 0.92$ & $7,327 \pm 1,679$ & 60.40 / 62.00 / 62.00 \\
 & {\proj} & $79.13 \pm 1.40$ & $460 \pm 103$ & 79.00 / 80.60 / 77.80 \\
\bottomrule
\end{tabular}
}
\end{table}

\begin{table}[H]
\centering
\caption{Three-run results on LoCoMo. FAIR-RAG and {\proj} (RAG) share the same RAG retriever; {\proj} (Zep) is compared with its Zep~\mbox{backbone}.}
\label{tab:fair_locomo}
\resizebox{\textwidth}{!}{
\begin{tabular}{l|cc|cc}
\toprule
\multirow{2}{*}{\textbf{Method}} & \multicolumn{2}{c|}{\textbf{GPT-4o-mini}} & \multicolumn{2}{c}{\textbf{GPT-4.1-mini}} \\ \cmidrule{2-5}
 & \textbf{J (\%)} & \textbf{Completion tokens} & \textbf{J (\%)} & \textbf{Completion tokens} \\
\midrule
FAIR-RAG & $68.96\pm0.23$ & $891\pm5$ & $78.46\pm0.20$ & $786\pm2$ \\
{\proj} (RAG) & $81.55\pm0.34$ & $300\pm6$ & $86.75\pm0.46$ & $431\pm1$ \\
\midrule
Zep & $74.69\pm0.24$ & $45\pm2$ & $78.94\pm0.15$ & $64\pm2$ \\
{\proj} (Zep) & $76.26\pm0.34$ & $390\pm1$ & $80.88\pm0.46$ & $496\pm4$ \\
\bottomrule
\end{tabular}
}
\end{table}

\begin{table}[H]
\centering
\caption{\label{tab:repeated}Individual {\proj} runs on LoCoMo, with category means and sample standard deviations. MH = multi-hop; T = temporal; OD = open-domain; SH = single-hop.}
\resizebox{\textwidth}{!}{
\begin{tabular}{c|c|c|llll|l}
\toprule
\textbf{LLM} & \textbf{Retriever} & \textbf{Run} & \textbf{MH} & \textbf{T} & \textbf{OD} & \textbf{SH} & \textbf{Overall} \\
\midrule\midrule
\multirow{8}{*}{\rotatebox[origin=c]{90}{\textbf{GPT-4o-mini}}}
& \multirow{4}{*}{Zep} & 1 & 68.09 & 73.52 & 68.75 & 80.72 & 76.13 \\
&                      & 2 & 69.86 & 72.59 & 67.71 & 80.36 & 76.00 \\
&                      & 3 & 70.21 & 75.39 & 64.58 & 80.72 & 76.65 \\
\cmidrule(lr){3-8}
& & Mean $\pm$ SD & 69.39 {\footnotesize $\pm$ 1.14} & 73.83 {\footnotesize$\pm$ 1.43} & 67.01 {\footnotesize$\pm$ 2.17} & 80.60 {\footnotesize$\pm$ 0.21} & 76.26 {\footnotesize$\pm$ 0.34} \\
\cmidrule(lr){2-8}
& \multirow{4}{*}{RAG} & 1 & 71.63 & 77.26 & 61.46 & 89.28 & 81.75 \\
&                      & 2 & 70.21 & 76.01 & 59.38 & 89.40 & 81.16 \\
&                      & 3 & 72.34 & 75.39 & 62.50 & 89.64 & 81.75 \\
\cmidrule(lr){3-8}
& & Mean $\pm$ SD & 71.39 {\footnotesize $\pm$ 1.08} & 76.22 {\footnotesize$\pm$ 0.95} & 61.11 {\footnotesize$\pm$ 1.59} & 89.44 {\footnotesize$\pm$ 0.18} & 81.55 {\footnotesize$\pm$ 0.34} \\
\midrule\midrule
\multirow{8}{*}{\rotatebox[origin=c]{90}{\textbf{GPT-4.1-mini}}}
& \multirow{4}{*}{Zep} & 1 & 78.72 & 78.50 & 72.92 & 84.34 & 81.36 \\
&                      & 2 & 75.89 & 77.26 & 68.75 & 84.58 & 80.44 \\
&                      & 3 & 78.72 & 77.57 & 67.71 & 84.34 & 80.84 \\
\cmidrule(lr){3-8}
& & Mean $\pm$ SD & 77.78 {\footnotesize $\pm$ 1.63} & 77.78 {\footnotesize$\pm$ 0.65} & 69.79 {\footnotesize$\pm$ 2.76} & 84.42 {\footnotesize$\pm$ 0.14} & 80.88 {\footnotesize$\pm$ 0.46} \\
\cmidrule(lr){2-8}
& \multirow{4}{*}{RAG} & 1 & 81.56 & 83.18 & 69.79 & 91.93 & 86.79 \\
&                      & 2 & 82.62 & 80.69 & 75.00 & 92.65 & 87.18 \\
&                      & 3 & 79.43 & 82.55 & 69.79 & 91.93 & 86.27 \\
\cmidrule(lr){3-8}
& & Mean $\pm$ SD & 81.20 {\footnotesize $\pm$ 1.62} & 82.14 {\footnotesize$\pm$ 1.29} & 71.53 {\footnotesize$\pm$ 3.01} & 92.17 {\footnotesize$\pm$ 0.42} & 86.75 {\footnotesize$\pm$ 0.46} \\
\bottomrule
\end{tabular}
}
\end{table}

\subsection{Category Breakdown on \texorpdfstring{LongMemEval$_s$}{LongMemEval-s}}
Table~\ref{tab:appdx_longmemeval} expands Fig.~\ref{fig:longmemeval} by category. {\proj} improves over Full-Context in most settings, particularly single-session preference. With GPT-5-mini, the largest deficits occur in temporal reasoning and multi-session questions, consistent with loss of useful details during compression. Table~\ref{tab:backend_strength} extends this backend-capability comparison to GPT-5.6 Terra.

\begin{table}[!t]
\centering
\caption{\label{tab:appdx_longmemeval}J scores (\%) on LongMemEval$_s$. P/A/U = single-session preference/assistant/user; TR = temporal reasoning; MS = multi-session; KU = knowledge update. Improvements are absolute differences from Full-Context. Three-run results are in Tables~\ref{tab:baseline_longmem} and~\ref{tab:fair_longmem}.}
\resizebox{\textwidth}{!}{
\begin{tabular}{l|l|cccccc|c}
\toprule
\textbf{Backend} & \textbf{Method} & \textbf{P} & \textbf{A} & \textbf{TR} & \textbf{MS} & \textbf{KU} & \textbf{U} & \textbf{Overall} \\
\midrule
\multirow{6}{*}[-0.25em]{GPT-4o-mini} & RAG & 56.67 & 94.64 & 42.11 & 33.83 & 70.51 & 87.14 & 57.40 \\
 & Self-RAG & 26.67 & 83.93 & 18.05 & 18.80 & 66.67 & 78.57 & 42.20 \\
 & RCR & 53.33 & 92.86 & 34.59 & 38.35 & 71.79 & 84.29 & 56.00 \\
\cmidrule(l){2-9}
 & Full-Context & 40.00 & 94.64 & 39.10 & 45.11 & 76.92 & 88.57 & 59.80 \\
 & {\proj} & 60.00 & 92.86 & 55.64 & 67.67 & 75.64 & 95.71 & 72.00 \\
 & Improvement & 20.00 & -1.78 & 16.54 & 22.56 & -1.28 & 7.14 & 12.20 \\
\midrule
\multirow{6}{*}[-0.25em]{GPT-4.1-mini} & RAG & 73.33 & 98.21 & 45.86 & 36.84 & 69.23 & 90.00 & 60.80 \\
 & Self-RAG & 80.00 & 96.43 & 32.33 & 30.08 & 75.64 & 94.29 & 57.20 \\
 & RCR & 86.67 & 98.21 & 63.16 & 54.89 & 74.36 & 92.86 & 72.20 \\
\cmidrule(l){2-9}
 & Full-Context & 56.67 & 98.21 & 48.87 & 50.38 & 78.21 & 88.57 & 65.40 \\
 & {\proj} & 86.67 & 100.00 & 61.65 & 70.68 & 76.92 & 97.14 & 77.20 \\
 & Improvement & 30.00 & 1.79 & 12.78 & 20.30 & -1.29 & 8.57 & 11.80 \\
\midrule
\multirow{6}{*}[-0.25em]{GPT-5-mini} & RAG & 86.67 & 98.21 & 57.89 & 38.35 & 70.51 & 88.57 & 65.20 \\
 & Self-RAG & 63.33 & 96.43 & 46.62 & 53.38 & 74.36 & 91.43 & 65.60 \\
 & RCR & 83.33 & 100.00 & 70.68 & 62.41 & 75.64 & 92.86 & 76.40 \\
\cmidrule(l){2-9}
 & Full-Context & 76.67 & 98.21 & 82.71 & 84.21 & 85.90 & 98.57 & 87.20 \\
 & {\proj} & 86.67 & 100.00 & 69.17 & 78.20 & 85.90 & 97.14 & 82.60 \\
 & Improvement & 10.00 & 1.79 & -13.54 & -6.01 & 0.00 & -1.43 & -4.60 \\
\midrule
\multirow{6}{*}[-0.25em]{DeepSeek-V3.2} & RAG & 50.00 & 94.64 & 32.33 & 28.57 & 57.69 & 87.14 & 51.00 \\
 & Self-RAG & 20.00 & 58.93 & 6.77 & 11.28 & 29.49 & 48.57 & 24.00 \\
 & RCR & 33.33 & 62.50 & 20.30 & 22.56 & 50.00 & 80.00 & 39.40 \\
\cmidrule(l){2-9}
 & Full-Context & 46.67 & 92.86 & 35.34 & 41.35 & 70.51 & 90.00 & 57.20 \\
 & {\proj} & 86.67 & 98.21 & 69.92 & 72.18 & 83.33 & 98.57 & 80.80 \\
 & Improvement & 40.00 & 5.35 & 34.58 & 30.83 & 12.82 & 8.57 & 23.60 \\
\midrule
\multirow{6}{*}[-0.25em]{Claude-haiku-4.5} & RAG & 70.00 & 98.21 & 42.86 & 36.84 & 71.79 & 91.43 & 60.40 \\
 & Self-RAG & 56.67 & 92.86 & 24.81 & 28.57 & 69.23 & 87.14 & 51.00 \\
 & RCR & 53.33 & 78.57 & 40.60 & 48.12 & 65.38 & 85.71 & 57.80 \\
\cmidrule(l){2-9}
 & Full-Context & 46.67 & 82.14 & 26.32 & 36.84 & 70.51 & 62.86 & 48.60 \\
 & {\proj} & 83.33 & 96.43 & 66.17 & 69.92 & 84.62 & 97.14 & 78.80 \\
 & Improvement & 36.66 & 14.29 & 39.85 & 33.08 & 14.11 & 34.28 & 30.20 \\
\midrule
\multirow{6}{*}[-0.25em]{Gemini-2.5-flash} & RAG & 53.33 & 87.50 & 42.11 & 38.35 & 61.54 & 80.00 & 55.20 \\
 & Self-RAG & 16.67 & 89.29 & 12.03 & 24.06 & 50.00 & 75.71 & 39.00 \\
 & RCR & 26.67 & 98.21 & 33.08 & 23.31 & 53.85 & 90.00 & 48.60 \\
\cmidrule(l){2-9}
 & Full-Context & 40.00 & 100.00 & 49.62 & 61.65 & 78.21 & 92.86 & 68.40 \\
 & {\proj} & 80.00 & 98.21 & 63.91 & 70.68 & 89.74 & 95.71 & 79.00 \\
 & Improvement & 40.00 & -1.79 & 14.29 & 9.03 & 11.53 & 2.85 & 10.60 \\
\bottomrule
\end{tabular}
}
\end{table}

\subsection{Query Decomposition and Rewriting}\label{appdx:query_baselines}
{\proj} outperforms both query-only strategies across the three tested backends (Table~\ref{tab:query_baselines}), supporting the contribution of state-based control beyond query reformulation.

\begin{table}[H]
\centering
\caption{J scores (\%) for query decomposition and iterative rewriting on LongMemEval$_s$.}
\label{tab:query_baselines}
\small
\begin{tabular}{l|ccc}
\toprule
\textbf{Backend} & \textbf{Decomp-RAG} & \textbf{Rewrite-RAG} & \textbf{{\proj}} \\
\midrule
GPT-4o-mini & 62.00 & 63.40 & 72.00 \\
GPT-4.1-mini & 68.40 & 68.40 & 77.20 \\
GPT-5-mini & 79.00 & 76.40 & 82.60 \\
\bottomrule
\end{tabular}
\end{table}

\subsection{Query-Time Cost and Backend Capability}\label{appdx:cost}
Table~\ref{tab:total_tokens} reports query-time token counts over all LLM calls. With GPT-4o-mini, {\proj} uses 5.9\% of Full-Context's total tokens.

\begin{table}[tb]
\centering
\caption{Average query-time tokens per question on LongMemEval$_s$, including all controller, state-generation, and answer-generation calls. GPT-5-mini completion counts include internal reasoning.}
\label{tab:total_tokens}\label{tab:backend_strength}
\small
\begin{tabular}{l|l|ccc|c}
\toprule
\textbf{Backend} & \textbf{Method} & \textbf{Input} & \textbf{Completion} & \textbf{Total} & \textbf{J (\%)} \\
\midrule
\multirow{5}{*}{GPT-4o-mini} & RAG & 1,403 & 8 & 1,411 & 57.40 \\
 & Self-RAG & 3,461 & 159 & 3,619 & 42.20 \\
 & RCR & 4,044 & 106 & 4,150 & 56.00 \\
 & {\proj} & 6,195 & 314 & 6,509 & 72.00 \\
 & Full-Context & 109,624 & 12 & 109,636 & 59.80 \\
\midrule
\multirow{2}{*}{GPT-4.1-mini} & {\proj} & 6,596 & 486 & 7,082 & 77.20 \\
 & Full-Context & 109,624 & 11 & 109,634 & 65.40 \\
\midrule
\multirow{2}{*}{GPT-5-mini} & {\proj} & 6,886 & 5,065 & 11,951 & 82.60 \\
 & Full-Context & 109,624 & 447 & 110,071 & 87.20 \\
\midrule
\multirow{2}{*}{GPT-5.6 Terra} & {\proj} & 6,428 & 415 & 6,843 & 85.80 \\
 & Full-Context & 109,624 & 46 & 109,670 & 88.80 \\
\bottomrule
\end{tabular}
\end{table}

\textbf{Backend capability.} Full-Context overtakes {\proj} on GPT-5-mini and GPT-5.6 Terra, while {\proj} retains its input-token advantage (Table~\ref{tab:backend_strength}).

\subsection{Token Efficiency and State Compression on LoCoMo}
\begin{figure*}[htbp]
\centering
\includegraphics[width=\textwidth]{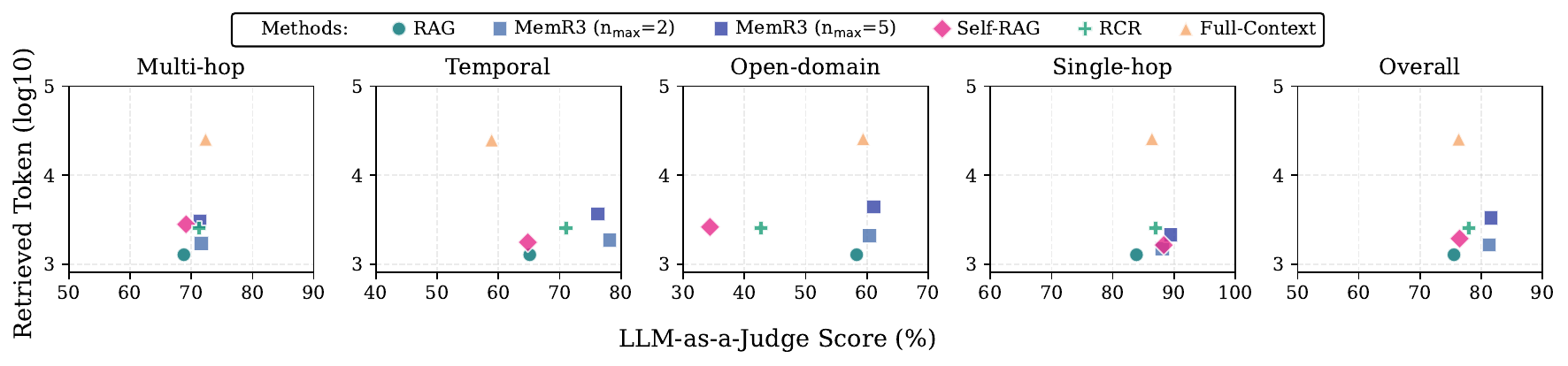}
\vspace{-10pt}
\caption{Average token usage of the retrieved context (y-axis) and LLM-as-a-Judge (J) Score (x-axis) of RAG, Self-RAG, RCR, {\proj}, and Full-Context across four categories on LoCoMo.}
\vspace{-10pt}
\label{fig:token_consume}
\end{figure*}

To empirically validate the structural advantages of our Markovian state machine, we analyze the token footprint versus reasoning accuracy. 

Unlike implicit agentic retrievers (Self-RAG, RCR) that append newly retrieved chunks to the context window, {\proj} compresses observations into an explicit state and masks assimilated raw text. As illustrated in Figure~\ref{fig:token_consume}, {\proj} occupies the high-accuracy, low-token region, improving the trade-off that constrains traditional autoregressive accumulators. On LoCoMo, $n_{\text{max}}=2$ uses fewer tokens and offers a favorable accuracy-cost trade-off, but exceeds $n_{\text{max}}=5$ in accuracy only for open-domain questions; the larger budget performs better in the other three categories.

\subsection{Lexical Metrics}\label{appdx:rebuttal_evaluation}
Table~\ref{tab:lexical_rebuttal} reports EM and F1 alongside the LLM-as-a-Judge score on LongMemEval$_s$. The default {\proj} does not constrain answer length and often produces explanatory answers, which overlap little with the short references. {\proj} (concise) reuses the same retrieval trajectory, evidence state, and draft answer, and only changes the final instruction to request an answer of at most six words. This raises EM/F1 from 0.00/17.72 to 19.00/43.49, while J decreases only from 72.00 to 68.40, indicating that the low lexical scores of the default setting largely reflect output format.

\begin{table}[H]
\centering
\caption{Auxiliary metrics (\%) on LongMemEval$_s$ with GPT-4o-mini.}
\label{tab:lexical_rebuttal}
\resizebox{\textwidth}{!}{
\begin{tabular}{l|cccc|cc}
\toprule
\textbf{Metric} & \textbf{RAG} & \textbf{Self-RAG} & \textbf{RCR} & \textbf{Full-Context} & \textbf{{\proj} (default)} & \textbf{{\proj} (concise)} \\
\midrule
EM & 23.60 & 0.20 & 0.00 & 24.20 & 0.00 & 19.00 \\
F1 & 43.21 & 14.28 & 13.90 & 46.48 & 17.72 & 43.49 \\
J & 57.40 & 42.20 & 56.00 & 59.80 & 72.00 & 68.40 \\
\bottomrule
\end{tabular}
}
\end{table}

\subsection{Convergence and Stopping}
\begin{figure}[htbp]
\centering
\includegraphics[width=\textwidth]{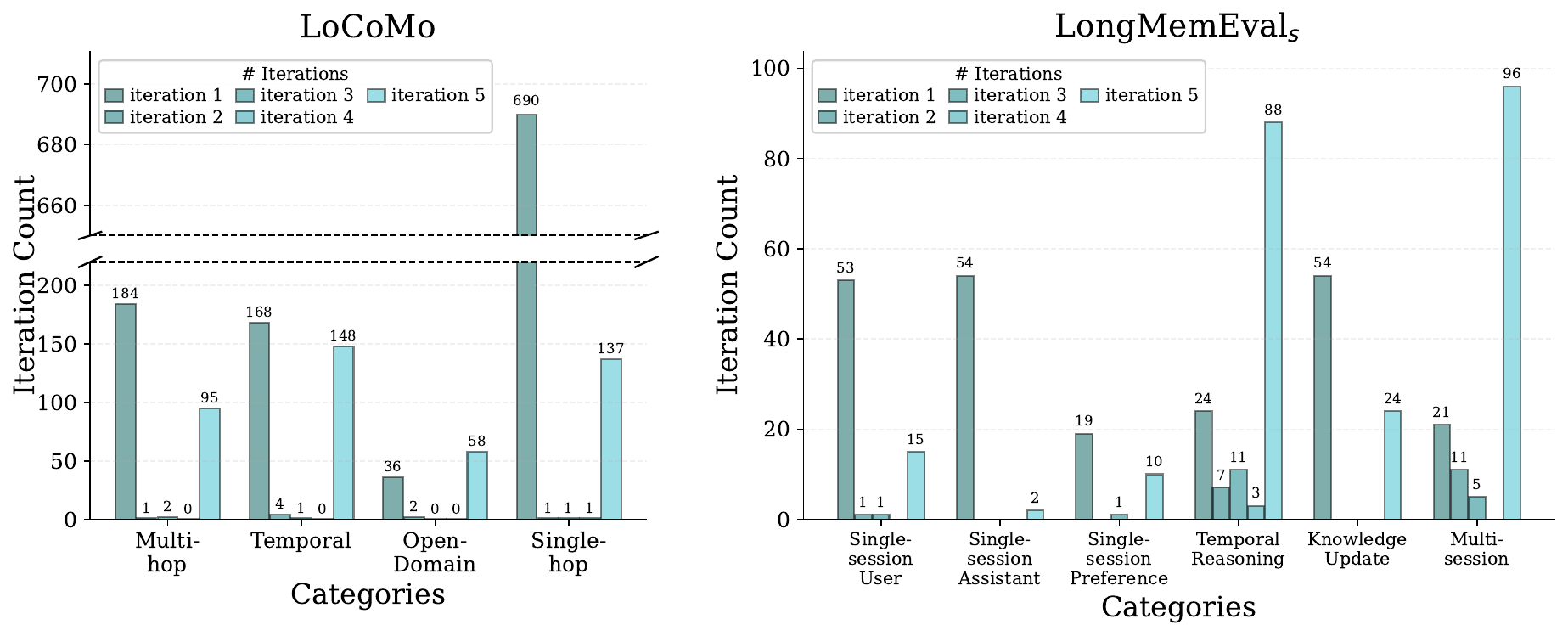}
\vspace{-10pt}
\caption{Number of questions requiring different numbers of iterations before final answers, across four categories on LoCoMo (left) and six categories on LongMemEval$_s$ (right).}
\label{fig:iteration_count}
\vspace{-10pt}
\end{figure}
We inspect the iteration frequency of {\proj} in Figure~\ref{fig:iteration_count}. Across both benchmarks, the distribution exhibits a clear bimodal pattern that is consistent with our evidence-gap controller. On \textbf{LoCoMo}, most queries terminate after either 1 iteration or at the maximum budget, with relatively few cases stopping in the middle. In particular, \textit{Single-hop} questions are overwhelmingly resolved in a single step, while \textit{Multi-hop}, \textit{Temporal}, and \textit{Open-Domain} questions show a larger fraction of budget-exhausted runs, reflecting cases where the gap cannot be fully closed from the retrieved evidence. On \textbf{LongMemEval$_s$}, the same pattern persists across all six categories: \textit{Single-session User}, \textit{Single-session Assistant}, \textit{Single-session Preference}, and \textit{Knowledge Update} are dominated by 1-iteration termination, whereas \textit{Temporal Reasoning} and especially \textit{Multi-session} contain many 5-iteration cases, indicating that these categories more often require repeated retrieval or remain unresolved under the available evidence.

\textbf{Stopping policies.} Stopping after two or three unchanged states reduces accuracy (Table~\ref{tab:stall_rebuttal}); a temporary stall can still be followed by useful retrieval. This trajectory-dependent rule differs from assigning a common iteration budget to all questions, as in Fig.~\ref{fig:token_consume}.

\begin{table}[H]
\centering
\caption{J scores (\%) on LongMemEval$_s$ with GPT-4o-mini. Stall-$m$ stops after $m$ consecutive iterations without evidence-state updates.}
\label{tab:stall_rebuttal}
\small
\begin{tabular}{l|cc|ccc}
\toprule
\textbf{Method} & \textbf{RAG} & \textbf{Full-Context} & \textbf{{\proj} (stall-2)} & \textbf{{\proj} (stall-3)} & \textbf{{\proj}} \\
\midrule
J (\%) & 57.40 & 59.80 & 66.00 & 67.40 & 72.00 \\
\bottomrule
\end{tabular}
\end{table}

\subsection{Case Studies.}\label{appdx:case_study}

\begin{figure}[tb!]
\centering
\includegraphics[width=\textwidth]{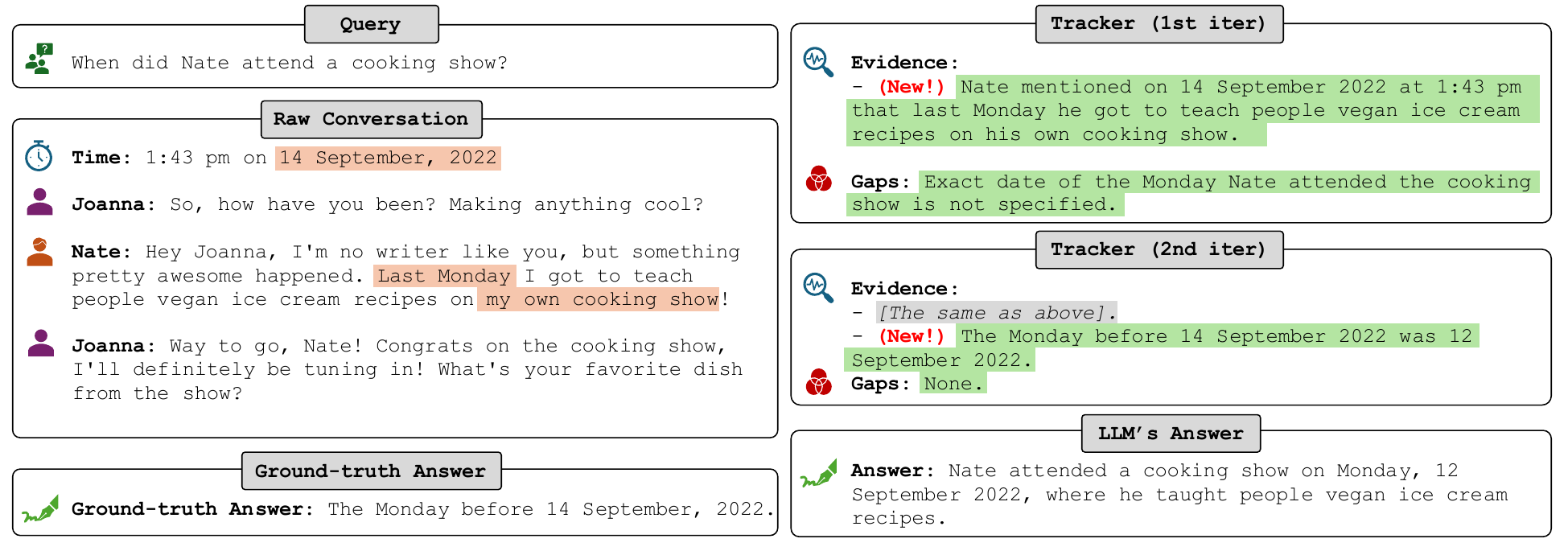}
\vspace{-10pt}
\caption{\label{fig:case-study-success}A successful case executed by {\proj} on LoCoMo. }
\vspace{-10pt}
\end{figure}

\textbf{Successful case.}
Fig.~\ref{fig:case-study-success} shows a successful case.
From the perspective of our idealized requirement space $R(q)$ in Appendix~\ref{sec:formal_eg}, this example corresponds to the case where the tracker is both sound and (approximately) complete. The requirement set for the query ``\textit{When did Nate attend a cooking show?}'' can be abstracted into two atomic requirements: the anchor timestamp of Nate's message on 14 September 2022 and the mapping from ``\textit{last Monday}'' to an absolute calendar date. After the first iteration, the tracker correctly records only the anchor requirement as evidence and keeps the exact date as a remaining gap, consistent with the monotonic update rule where evidence grows while gaps shrink. Once the second snippet explicitly states that ``\textit{The Monday before 14 September 2022 was 12 September 2022}'', the date requirement becomes supported, and the gap set becomes empty, illustrating the completeness property of the ideal tracker and showing that {\proj} can closely approximate this behavior in practice.


Next, we also present three types of failed (or structurally non-convergent) cases: coverage bias, hallucination bias, and granularity bias.

\textbf{Failed case with coverage bias (false negatives).} 
Fig.~\ref{fig:case-study-coverage} shows a failed case with coverage bias.
Here, the requirement space $R(q)$ for the query ``\textit{In what activity did Maria and her church friends participate in July 2023?}'' contains an atomic requirement $r_{\text{act}}$ corresponding to the activity ``\textit{hiking}''. The raw memory snippet explicitly supports $r_{\text{act}}$, so under the ideal update rule, we would have $m \vDash r_{\text{act}}$ for some retrieved memory item $m\in \mathcal{P}_{k-1}$, and thus $r_{\text{act}} \in \mathcal{E}_k^\star$ once that snippet is retrieved. However, in this specific iteration, the tracker failed to extract the ``\textit{hiking}'' activity due to interference from other activities. Consequently, the practical tracker could not include this requirement in the evidence set and instead focused on the retrieved picnic-related facts, so $r_{\text{act}} \notin \mathcal{E}_k$. This is exactly a coverage bias (false negative): a supported requirement in $R(q)$ is omitted from the evidence, leading to an incorrect answer.

\begin{figure}[tb!]
\centering
\includegraphics[width=\textwidth]{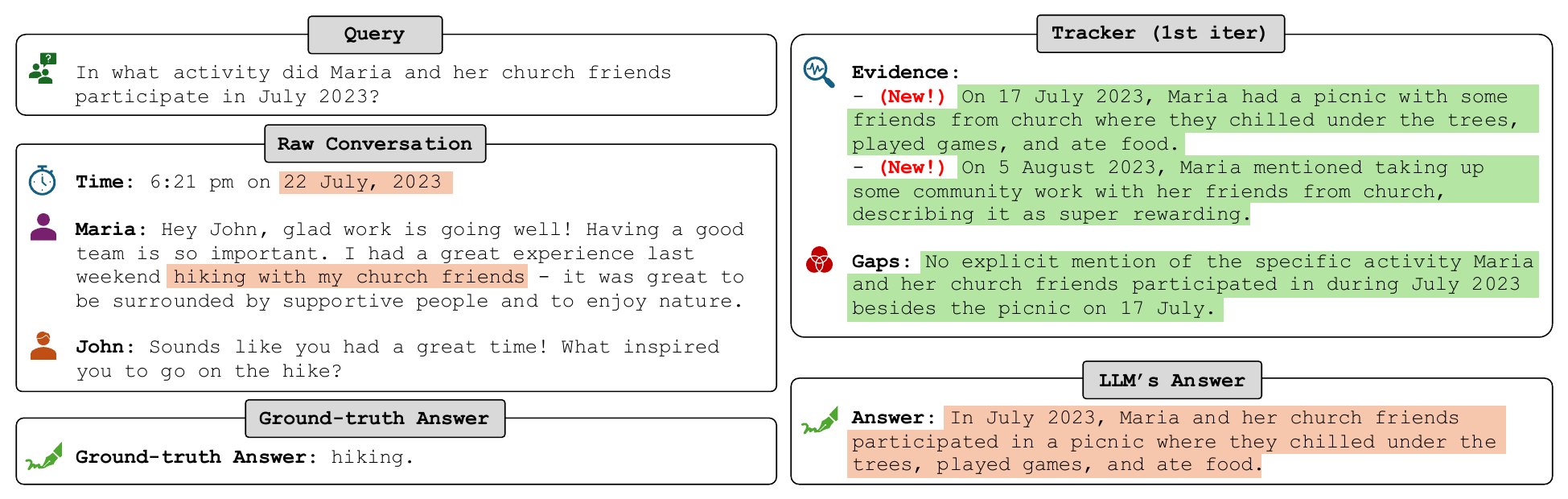}
\caption{\label{fig:case-study-coverage}A failed case executed by {\proj} with coverage bias (false negatives). Although the retrieved snippets contain evidence supporting the requirement,
the LLM-based tracker fails to extract and record it in $\mathcal E_k$.}
\end{figure}

\textbf{Failed case with hallucination bias (false positives).} 
Fig.~\ref{fig:case-study-hallucination}  shows a failed case with hallucination bias, 
which illustrates the complementary failure mode where the tracker adds unsupported content to the evidence. For the query ``\textit{When did Andrew go rock climbing?}'', the requirement space $R(q)$ includes a date requirement $r_{\text{date}}$ for the rock-climbing event. The retrieved snippet only supports a weaker predicate stating that the event happened on ``\textit{last Sunday}'' relative to 13 June 2023, i.e., it provides a relative temporal relation but not a fully grounded calendar date. Under the idealized semantics, no memory item satisfies $m \vDash r_{\text{date}}(\text{2023-06-13})$; thus, $r_{\text{date}}$ should remain in the gap set until additional information is retrieved. In contrast, the practical tracker directly records ``\textit{Andrew went rock climbing on 13 June, 2023}'' as evidence and sets the gap to ``\textit{None}'', introducing a requirement in $\mathcal{E}_k$ that is not supported by any memory item. This is precisely a hallucination bias (a false positive), and it explains why the final answer overspecifies the date to 13 June rather than the correct 11 June.

\begin{figure}[tb!]
\centering
\includegraphics[width=\textwidth]{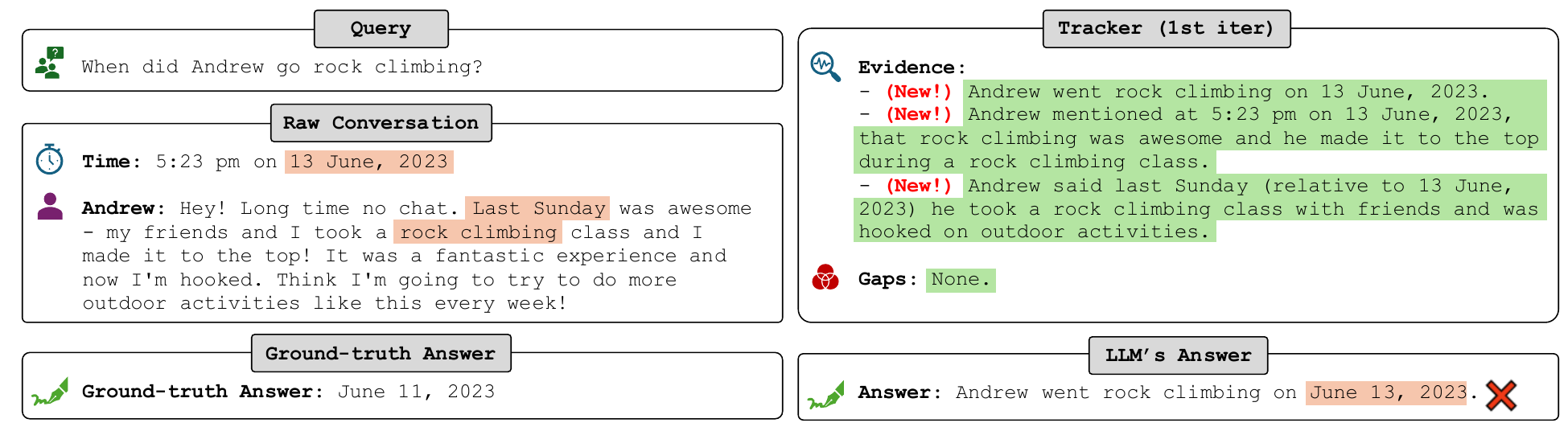}
\caption{\label{fig:case-study-hallucination}A failed case executed by {\proj} with hallucination bias (false positives).}
\end{figure}

\textbf{Case with granularity bias.}
In Fig.~\ref{fig:case-study-granularity}, for the question ``\textit{When did Melanie paint a sunrise?}'', the correct answer is simply ``\textit{2022}'' (the year). 
Under the ideal abstraction, however, the requirement space $R(q)$ implicitly contains a fine-grained predicate $r_{\text{date}}$ corresponding to the full year-month-day of the painting event. The memory store only contains a coarse statement such as ``\textit{Melanie painted the lake sunrise image last year (2022).}'' 

In the ideal tracker, no memory item $m$ satisfies $m \models r_{\text{date}}$, so the precondition of Theorem~\ref{thm:eg_properties}'s completeness clause is violated and the ideal gap $\mathcal{G}_k$ never becomes empty. The practical LLM tracker mirrors this behavior: it quickly recovers the year 2022 as evidence, but continues to treat the exact date as a remaining gap, eventually hitting the iteration budget without fully closing $\mathcal{G}_k$. This example shows that some apparent ``failures'' of the approximate tracker are in fact structural: they arise from a mismatch between the granularity of $R(q)$ and the information actually present in the memory store.

Based on the above case studies, we suggest that future studies focus on relieving the aforementioned biases (coverage, hallucination, granularity) to keep memory systems accurate and efficient.

\begin{figure}[tb!]
\centering
\includegraphics[width=\textwidth]{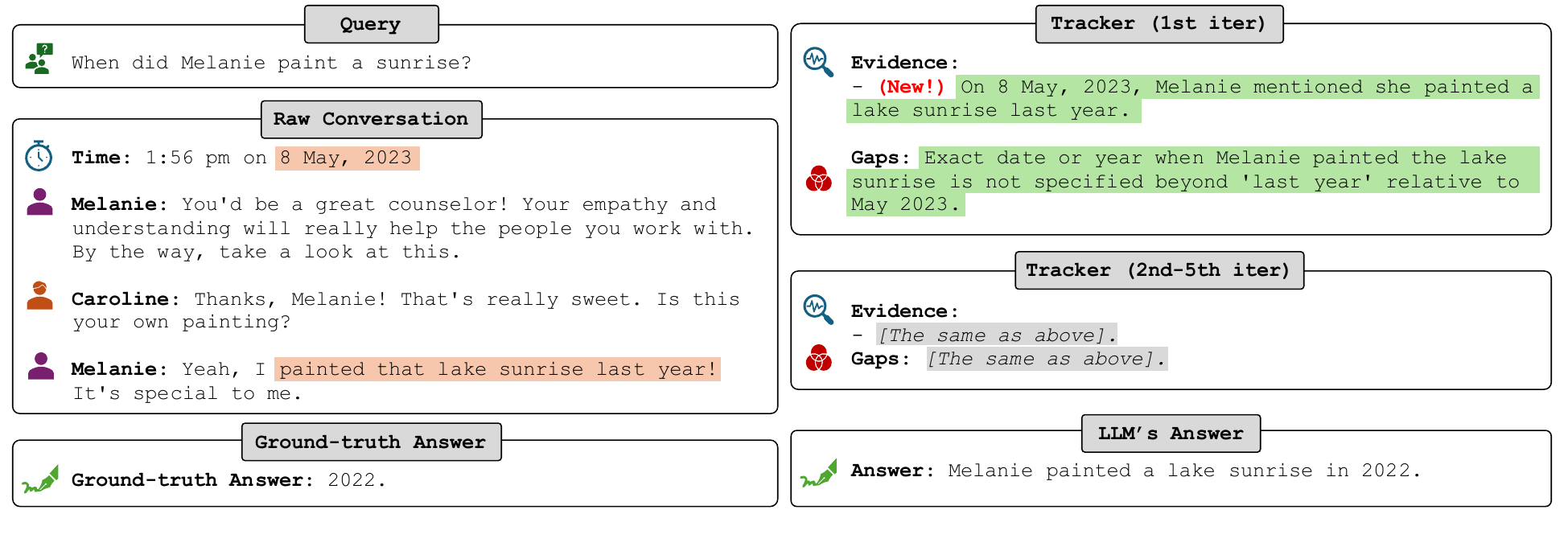}
\caption{\label{fig:case-study-granularity}A case executed by {\proj} with granularity bias. Although the final answer is correct, it cannot close the loop until the max iteration is reached.}
\end{figure}

\section{Re-alignment of LoCoMo dataset}\label{appdx:data_misalign}
\textbf{Misalignment in existing works.}
Although the correct order of the different categories is not explicitly reported in LoCoMo~\citep{maharana2024evaluating}, we can infer it from the number of questions in each category. The correct alignment is shown in Table~\ref{tab:locomo_alignment}.
We believe this clarification could benefit the LLM memory community.

\begin{table}[H]
\centering
\caption{\label{tab:locomo_alignment}The alignment of the category orders and names in the LoCoMo dataset.
}
\resizebox{0.8\textwidth}{!}{
\begin{tabular}{c|ccccc}
\toprule
Category & Category 1 & Category 2 & Category 3 & Category 4 & Category 5 \\ \midrule
Name & Multi-Hop & Temporal & Open-Domain & Single-Hop & Adversarial \\
\midrule
\# Questions & 282 & 321 & 96 & 830 & 445\\
\bottomrule
\end{tabular}
}
\end{table}

\newpage
\section*{NeurIPS Paper Checklist}

\begin{enumerate}

\item {\bf Claims}
    \item[] Question: Do the main claims made in the abstract and introduction accurately reflect the paper's contributions and scope?
    \item[] Answer: \answerYes{} 
    \item[] Justification: Yes. The abstract and introduction accurately summarize the paper's proposed evidence-gap controller, its scope in conversational memory, and the empirical efficiency-accuracy claims supported by the experiments.

\item {\bf Limitations}
    \item[] Question: Does the paper discuss the limitations of the work performed by the authors?
    \item[] Answer: \answerYes{} 
    \item[] Justification: Yes. The paper explicitly discusses limitations including dependence on backend recall, LLM state-estimation errors, limited large-scale backend comparisons, and restriction to text-only memory.

\item {\bf Theory assumptions and proofs}
    \item[] Question: For each theoretical result, does the paper provide the full set of assumptions and a complete (and correct) proof?
    \item[] Answer: \answerYes{} 
    \item[] Justification: Yes. The theoretical claims are stated under an explicit idealized abstraction, with assumptions and formal proofs provided in the appendix.

    \item {\bf Experimental result reproducibility}
    \item[] Question: Does the paper fully disclose all the information needed to reproduce the main experimental results of the paper to the extent that it affects the main claims and/or conclusions of the paper (regardless of whether the code and data are provided or not)?
    \item[] Answer: \answerYes{} 
    \item[] Justification: Yes. The paper provides the datasets, baselines, metrics, prompts, model backends, and key hyperparameters needed to reproduce the main experimental results.

\item {\bf Open access to data and code}
    \item[] Question: Does the paper provide open access to the data and code, with sufficient instructions to faithfully reproduce the main experimental results, as described in supplemental material?
    \item[] Answer: \answerYes{} 
    \item[] Justification: Yes. An anonymized code link is released.

\item {\bf Experimental setting/details}
    \item[] Question: Does the paper specify all the training and test details (e.g., data splits, hyperparameters, how they were chosen, type of optimizer) necessary to understand the results?
    \item[] Answer: \answerYes{} 
    \item[] Justification: Yes. The paper specifies the datasets, baselines, evaluation metrics, model backends, and main controller settings used in the experiments.

\item {\bf Experiment statistical significance}
    \item[] Question: Does the paper report error bars suitably and correctly defined or other appropriate information about the statistical significance of the experiments?
    \item[] Answer: \answerYes{} 
    \item[] Justification: Yes. The paper reports the averaged runs and standard deviations on LoCoMo datasets.

\item {\bf Experiments compute resources}
    \item[] Question: For each experiment, does the paper provide sufficient information on the computer resources (type of compute workers, memory, time of execution) needed to reproduce the experiments?
    \item[] Answer: \answerYes{} 
    \item[] Justification: Yes. The computer resources are shown in the appendix.
    
\item {\bf Code of ethics}
    \item[] Question: Does the research conducted in the paper conform, in every respect, with the NeurIPS Code of Ethics \url{https://neurips.cc/public/EthicsGuidelines}?
    \item[] Answer: \answerYes{} 
    \item[] Justification: To the best of our knowledge, the research complies with the NeurIPS Code of Ethics and does not involve human-subject experimentation or other ethically sensitive procedures.

\item {\bf Broader impacts}
    \item[] Question: Does the paper discuss both potential positive societal impacts and negative societal impacts of the work performed?
    \item[] Answer: \answerYes{} 
    \item[] Justification: The paper briefly discusses the broader impacts in the conclusion.
    
\item {\bf Safeguards}
    \item[] Question: Does the paper describe safeguards that have been put in place for responsible release of data or models that have a high risk for misuse (e.g., pre-trained language models, image generators, or scraped datasets)?
    \item[] Answer: \answerNA{} 
    \item[] Justification: The paper does not release a new high-risk model or scraped dataset that would require special release safeguards.

\item {\bf Licenses for existing assets}
    \item[] Question: Are the creators or original owners of assets (e.g., code, data, models), used in the paper, properly credited and are the license and terms of use explicitly mentioned and properly respected?
    \item[] Answer: \answerYes{} 
    \item[] Justification: Yes. The paper cites the external datasets and systems used and explicitly documents the licenses.

\item {\bf New assets}
    \item[] Question: Are new assets introduced in the paper well documented and is the documentation provided alongside the assets?
    \item[] Answer: \answerNA{} 
    \item[] Justification: The paper does not introduce or release a new dataset, model, or code asset as part of the submission.

\item {\bf Crowdsourcing and research with human subjects}
    \item[] Question: For crowdsourcing experiments and research with human subjects, does the paper include the full text of instructions given to participants and screenshots, if applicable, as well as details about compensation (if any)? 
    \item[] Answer: \answerNA{} 
    \item[] Justification: The paper does not involve crowdsourcing or experiments with human subjects.

\item {\bf Institutional review board (IRB) approvals or equivalent for research with human subjects}
    \item[] Question: Does the paper describe potential risks incurred by study participants, whether such risks were disclosed to the subjects, and whether Institutional Review Board (IRB) approvals (or an equivalent approval/review based on the requirements of your country or institution) were obtained?
    \item[] Answer: \answerNA{} 
    \item[] Justification: The paper does not involve human-subject research requiring IRB approval or an equivalent review.

\item {\bf Declaration of LLM usage}
    \item[] Question: Does the paper describe the usage of LLMs if it is an important, original, or non-standard component of the core methods in this research? Note that if the LLM is used only for writing, editing, or formatting purposes and does \emph{not} impact the core methodology, scientific rigor, or originality of the research, declaration is not required.
    \item[] Answer: \answerYes{} 
    \item[] Justification: The paper clearly states that LLMs are a core component of the method, serving as the state estimator, router, and answer generator.

\end{enumerate}

\end{document}